\documentclass{article}

\usepackage{microtype}
\usepackage{graphicx}
\usepackage{subcaption}
\usepackage{booktabs} 
\usepackage{multicol}
\usepackage[dvipsnames]{xcolor}

\usepackage{hyperref}
\usepackage{amsmath}
\usepackage{amssymb}
\usepackage{mathtools}
\usepackage{natbib}
\usepackage{enumerate}
\usepackage{tikz}
\usepackage{graphicx}

\DeclarePairedDelimiter\norm{\lVert}{\rVert}%

\newcommand{\textdiff}[1]{\textcolor{red}{#1}}

\newcommand{\normt}[1]{\left\lVert#1\right\rVert_2}
\newcommand{\normtmu}[1]{\left\lVert#1\right\rVert_{2, \mu}}

\newcommand{\normtt}[1]{\left\lVert#1\right\rVert^2_2}

\newcommand{\defeq}{\mathrel{\stackrel{\makebox[0pt]{\mbox{\normalfont\tiny def}}}{=}}}

\newcommand{\argmax}[1]{\underset{#1}{\textrm{argmax}}\ }
\newcommand{\argmin}[1]{\underset{#1}{\textrm{argmin}}\ }

\newcommand{\Projmu}{\Pi_\mu}
\newcommand{\trans}{T}
\newcommand{\backup}{\mathcal{T}}
\newcommand{\Qclass}{\mathcal{Q}}
\newcommand{\ReplayBuffer}{\mathcal{B}}
\newcommand{\ltwonorm}{\ell^2}
\newcommand{\linfnorm}{\ell^\infty}

\usepackage{semicrunch}
\usepackage{wrapfig,lipsum,booktabs}

\usepackage{amsthm}
\newtheorem{theorem}{Theorem}[section]

\newtheorem{corollary}{Corollary}[theorem]

\usepackage[accepted]{icml2019}

\begin{document}

\twocolumn[
\icmltitle{Diagnosing Bottlenecks in Deep Q-learning Algorithms}



\icmlsetsymbol{equal}{*}

\begin{icmlauthorlist}
\icmlauthor{Justin Fu}{equal,to}
\icmlauthor{Aviral Kumar}{equal,to}
\icmlauthor{Matthew Soh}{to}
\icmlauthor{Sergey Levine}{to}
\end{icmlauthorlist}

\icmlaffiliation{to}{Berkeley AI Research, University of California, Berkeley}

\icmlcorrespondingauthor{Justin Fu}{\texttt{justinjfu@eecs.berkeley.edu}}
\icmlcorrespondingauthor{Aviral Kumar}{\texttt{aviralk@berkeley.edu}}


\vskip 0.3in
]



\printAffiliationsAndNotice{\icmlEqualContribution}  


\begin{abstract}
Q-learning methods represent a commonly used class of algorithms in reinforcement learning: they are generally efficient and simple, and can be combined readily with function approximators for deep reinforcement learning (RL). However, the behavior of Q-learning methods with function approximation is poorly understood, both theoretically and empirically. In this work, we aim to experimentally investigate potential issues in Q-learning, by means of a "unit testing" framework where we can utilize oracles to disentangle sources of error. 
Specifically, we investigate questions related to function approximation, sampling error and nonstationarity, and where available, verify if trends found in oracle settings hold true with modern deep RL methods.
We find that large neural network architectures have many benefits with regards to learning stability; offer several practical compensations for overfitting; and develop a novel sampling method based on explicitly compensating for function approximation error that yields fair improvement on high-dimensional continuous control domains. 
\end{abstract}

\section{Introduction}
Q-learning algorithms, which are based on approximating state-action value functions, are an efficient and commonly used class of RL methods. In recent years, such methods have been applied to great effect in domains such as playing video games from raw pixels~\citep{Mnih2015} and continuous control in robotics~\citep{kalashnikov18}. Methods based on approximate dynamic programming and Q-function estimation have several very appealing properties: they are generally moderately sample-efficient, when compared to policy gradient methods, they are simple to use, and they allow for off-policy learning. This makes them an appealing choice for a wide range of tasks, from robotic control~\citep{kalashnikov18} to off-policy learning from historical data for recommender~\citep{shani2005recommender} systems and other applications. However, although the basic tabular Q-learning algorithm is convergent and admits theoretical analysis~\cite{suttonrlbook}, its non-linear counterpart with function approximation (such as with deep neural networks) is poorly understood theoretically. 
In this paper, we aim to investigate the degree to which the theoretical issues with Q-learning actually manifest in practice. 
Thus, we empirically analyze aspects of the Q-learning method in a \emph{unit testing} framework, where we can employ oracle solvers to obtain ground truth Q-functions and distributions for exact analysis. We investigate the following questions:

\textbf{1) What is the effect of function approximation on convergence?}
Most practical reinforcement learning problems, such as robotic control, require function approximation to handle large or continuous state spaces. However, the behavior of Q-learning methods under function approximation is not well understood. There are known counterexamples where the method diverges~\citep{Baird1995}, and there are no known convergence guarantees~\citep{suttonrlbook}. 
To investigate these problems, we study the convergence behavior of Q-learning methods with function approximation, parametrically varying the function approximator power and analyzing the quality of the solution as compared to the optimal Q-function and the optimal projected Q-function under that function approximator. 
{We find, somewhat surprisingly, that function approximation error is not a major problem in Q-learning algorithms, but only when the representational capacity of the function approximator is high. This makes sense in light of the theory: a high-capacity function approximator can perform a nearly perfect projection of the backed up Q-function, thus mitigating potentially convergence issues due to an imperfect $\ell_2$ norm projection. We also find that divergence rarely occurs, for example, we observed divergence in only 0.9\% of our experiments. We discuss this further in Section~\ref{sec:function_approx}.}

\textbf{2) What is the effect of sampling error and overfitting?}
Q-learning is used to solve problems where we do not have access to the transition function of the MDP. Thus, Q-learning methods need to learn by collecting samples in the environment, and training on these samples incurs sampling error, potentially leading to overfitting. This causes errors in the computation of the Bellman backup, which degrades the quality of the solution. {We experimentally show that overfitting exists in practice by performing ablation studies on the number of gradient steps, and by demonstrating that oracle based early stopping techniques can be used to improve performance of Q-learning algorithms.}  (Section~\ref{sec:overfitting}).
Thus, in our experiments we quantify the amount of overfitting which happens in practice, incorporating a variety of metrics, an performing a number of ablations and investigate methods to mitigate its effects.

\textbf{3) What is the effect of distribution shift and a moving target?}
The standard formulation of Q-learning prescribes an update rule, with no corresponding objective function~\citep{Sutton09b}. This results in a process which optimizes an objective that is non-stationary in two ways: the target values are updated during training, and the distribution under which the Bellman error is optimized changes, as samples are drawn from different policies. We refer to these problems as the \emph{moving target} and \emph{distribution shift} problems, respectively. These properties can make convergence behavior difficult to understand, and prior works have hypothesized that nonstationarity is a source of instability~\citep{Mnih2015, Lillicrap2015}. {In our experiments, we develop metrics to quantify the amount of distribution shift and performance change due to non-stationary targets. Surprisingly, we find that in a controlled experiment, distributional shift and non-stationary targets do not in fact correlate with reduction in performance. In fact, sampling strategies with large distributional shift often perform very well.}

\textbf{4) What is the best sampling or weighting distribution?}
Deeply tied to the distribution shift problem is the choice of which distribution to sample from. Do moving distributions cause instability, as Q-values trained on one distribution are evaluated under another in subsequent iterations?
Researchers have often noted that on-policy samples are typically superior to off-policy samples~\citep{suttonrlbook}, and there are several theoretical results that highlight favorable convergence properties under on-policy samples. However, there is little theoretical guidance on how to pick distributions so as to maximize learning rate. To this end, we investigate several choices for the sampling distribution. {Surprisingly, we find that on-policy training distributions are not always preferable, and that a clear pattern in performance with respect to training distribution is that broader, higher-entropy distributions perform better, regardless of distributional shift. Motivated by our findings, we propose a novel weighting distribution, adversarial feature matching (AFM), which is explicitly compensates for function approximator error, while still producing high-entropy sampling distributions.}

Our contributions are as follows:
We introduce a unit testing framework for Q-learning to disentangle issues related to function approximation, sampling, and distributional shift where approximate components are replaced by oracles. This allows for controlled analysis of different sources of error. We perform a detailed experimental analysis of many hypothesized sources of instability, error, and slow training in Q-learning algorithms on tabular domains, and show that many of these trends hold true in high dimensional domains. We propose novel choices of sampling distributions which lead to improved performance even on high-dimensional tasks. Our overall aim is to offer practical guidance for designing RL algorithms, as well as to identify important issues to solve in future research.

\section{Preliminaries}
\label{sec:backrgound}
Q-learning algorithms aim to solve a Markov decision process (MDP) by learning the optimal state-action value function, or Q-function. We define an MDP as a tuple $(\mathcal{S}, \mathcal{A}, \trans, R, \gamma)$. $\mathcal{S}, \mathcal{A}$ represent the state and action spaces, respectively. $\trans(s' | s, a)$ and $R(s,a)$ represent the dynamics (transition distribution) and reward function, and $\gamma \in (0,1)$ represents the discount factor. The goal in RL is to find a policy $\pi(a|s)$ that maximizes the expected cumulative discounted rewards, known as the \textit{returns}:
\[ \pi^* = \argmax{\pi} E_{s_{t+1} \sim \trans(\cdot|s_t, a_t), a_t \sim \pi(\cdot|s_t)}\left[\sum_{t=0}^\infty \gamma^t R(s_t, a_t)\right] \]
The quantity of interest in Q-learning methods are state-action value functions, which give the expected future return starting from a particular state-action tuple, denoted $Q^\pi(s,a)$. The state value function can also be denoted as $V^\pi(s,a)$. Q-learning algorithms are based on iterating the Bellman backup operator $\backup$, defined as
\[(\backup Q)(s, a) = R(s, a) + \gamma E_{s' \sim \trans}[V(s')]\]
\[V(s) = \max_{a'} Q(s, a')\]
The (tabular) Q-iteration algorithm is a dynamic programming algorithm that iterates the Bellman backup $Q^{t+1} \leftarrow \backup Q^t$. Because the Bellman backup is a $\gamma$-contraction in the L-$\infty$ norm, and $Q^*$ (the Q-values of $\pi^*$) is its fixed point, Q-iteration can be shown to converge to $Q^*$~\citep{suttonrlbook}. A deterministic optimal policy can then be obtained as $\pi^*(s) = \argmax{a} Q^*(s,a)$.

When state spaces cannot be enumerated in a tabular format, function approximators can be used to represent the Q-values. An important class of such Q-learning methods are \textit{fitted Q-iteration} (FQI)~\citep{Ernst05}, or approximate dynamic programming (ADP) methods, which form the basis of modern deep RL methods such as DQN~\citep{Mnih2015}.
FQI projects the values of the Bellman backup onto a family of Q-function approximators $\Qclass$:
\[ Q^{t+1} \leftarrow \Projmu(\backup Q^t) .\]
Here, $\Projmu$ denotes a $\mu$-weighted L2 projection, which minimizes the \textit{Bellman error} via supervised learning:
\begin{equation}
\label{eqn:bellman_projection} 
\Projmu(Q) \defeq 
\argmin{Q' \in \Qclass} E_{s,a \sim \mu}[(Q'(s,a) - Q(s,a))^2]
 .\end{equation}
The values produced by the Bellman backup, $(\backup Q^t)(s,a)$ are commonly referred to as \textit{target values}, and when neural networks are used for function approximation, the previous Q-function $Q^t(s,a)$ is referred to as the \textit{target network}. In this work, we distinguish between the cases when the Bellman error is estimated with Monte-Carlo sampling or computed exactly (see Section~\ref{sec:setup_algos}). The sampled variant corresponds to FQI as described in the literature~\citep{Ernst05,Riedmiller2005}, while the exact variant is analogous to conventional ADP~\citep{Bertsekas96}. 

Convergence guarantees for Q-iteration do not cleanly translate to FQI. $\Projmu$ is an $\ltwonorm$ projection, but $\backup$ is a contraction in the $\linfnorm$ norm -- this norm mistmatch means the composition of the backup and projection is no longer guaranteed to be a contraction under any norm~\citep{Bertsekas96}, and hence the convergence is not guaranteed.

A related branch of Q-learning methods are \textit{online Q-learning} methods,
in which Q-values are updated while samples are being collected in the MDP. This includes classic algorithms such as Watkin's Q-learning~\citep{Watkins1992}. Online Q-learning methods can be viewed as a form of stochastic approximation (such as Robbins-Monro) applied to Q-iteration and FQI~\citep{Bertsekas96}, and share many of its theoretical properties~\citep{szepesvari1998asymptotic}.
Modern deep RL algorithms such as DQN~\citep{Mnih2015} have characteristics of both online Q-learning and FQI -- using replay buffers means the sampling distribution $\mu$ changes very little between target updates (see Section~\ref{sec:distr_shift}), and target networks are justified from the viewpoint of FQI. Because FQI corresponds to the case when the sampling distribution is static between target updates, the behavior of modern deep RL methods more closely resembles FQI than a true online method without target networks.

\section{Experimental Setup}
\label{sec:setup}
Our experimental setup is centered around \emph{unit-testing}. We first introduce a spectrum of Q-learning algorithms, starting with exact approximate dynamic programming and gradually replacing oracle components, such as knowledge of dynamics, until the algorithm resembles modern deep Q-learning methods. We then introduce a suite of tabular environments where oracle solutions can be computed and compared against, to aid in diagnosis, as well as testing in high-dimensional environments to verify our hypotheses.

In order to provide consistent metrics across domains, we normalize returns and errors involving Q-functions (such as Bellman error) by the returns of the expert policy $\pi^*$ on each environment.

\subsection{Algorithms}
\label{sec:setup_algos}
{In the analysis presented in Section \ref{sec:function_approx}, \ref{sec:overfitting}, \ref{sec:analysis_nonstationarity} and \ref{sec:sampling_distributions}, we will use three different Q-learning variants, each of which remove some of the approximations in the standard Q-learning method used in the literature} -- 
Exact-FQI, Sampling-FQI, and Replay-FQI. Although FQI is not exactly identical to commonly used deep RL methods, such as DQN~\cite{Mnih2015}, DDPG~\cite{Lillicrap2015}, and SAC~\cite{Haarnoja2017}, it is structurally similar and, when the replay buffer for the commonly used methods becomes large, the difference becomes negligible, since the sampling distribution changes very little between target network updates. However, FQI methods are much more amenable for controlled analysis, since we can separately isolate target values, update rates, and the number of samples used for each iteration. We therefore use variants of FQI as the basis for our analysis, but we also confirm that similar trends hold with more commonly used algorithms on standard benchmark problems.

\begin{figure*}[ttt!]
\begin{small}
\begin{minipage}[t]{0.33\linewidth}
\begin{algorithm}[H]
\small
\caption{Exact-FQI}
\label{alg:fqiexact}
\begin{algorithmic}[1]
    \STATE Initialize Q-value approximator $Q_\theta(s,a)$.
    \FOR{step $t$ in \{1, \dots, N\}}
        \item[]
        \item[]
        \item[]
        \STATE Evaluate $Q_{\theta^t}(s,a)$ at all states.
        \STATE Compute exact target values at all states. \\
        $y(s,a) = r(s,a) + \gamma E_{s'}[ V_{\theta^t}(s')]$ 
        \STATE Minimize projection loss with respect to $\mu$: \\
        $\argmin{\theta} E_\mu[(Q_\theta(s,a) - y(s,a))^2]$
    \ENDFOR
\end{algorithmic}
\end{algorithm}
\end{minipage}
\begin{minipage}[t]{0.33\linewidth}
\begin{algorithm}[H]
\small
\caption{Sampled-FQI}
\label{alg:fqisampled}
\begin{algorithmic}[1]
    \STATE Initialize Q-value approximator $Q_\theta(s,a)$.
    \FOR{step $t$ in \{1, \dots, N\}}
        \item[]
        \item[]
        \STATE \textdiff{Collect $M$ samples from $\mu$.}
        \STATE Evaluate $Q_{\theta^t}(s,a)$ \textdiff{on samples.}
        \STATE Compute sampled target values \textdiff{on samples.}\\
        $\hat{y}_i = r_i + \gamma V_{\theta^t}(s'_i)$ 
        \STATE Minimize projection loss with respect to \textdiff{samples}: \\
        $ \argmin{\theta} \frac{1}{M}\sum_{i=1}^M (Q_\theta(s_i,a_i) - y_i)^2$
    \ENDFOR
\end{algorithmic}
\end{algorithm}
\end{minipage}
\begin{minipage}[t]{0.33\linewidth}
\begin{algorithm}[H]
\small
\caption{Replay-FQI}
\label{alg:fqireplay}
\begin{algorithmic}[1]
    \STATE Initialize Q-value approximator $Q_\theta(s,a)$, \textdiff{replay buffer $\ReplayBuffer$}.
    \FOR{step $t$ in \{1, \dots, N\}}
        \STATE \textdiff{Collect $K$ online samples from $\mu$.}
        \STATE \textdiff{Append online samples to buffer $\ReplayBuffer$.}
        \STATE Collect $M$ samples from $\ReplayBuffer$.
        \STATE Evaluate $Q_{\theta^t}(s,a)$ on samples.
        \STATE Compute sampled target values on samples\\
        $\hat{y}_i = r_i + \gamma V_{\theta^t}(s'_i)$ 
        \STATE Minimize projection loss with respect to samples: \\
        $ \argmin{\theta} \frac{1}{M}\sum_{i=1}^M (Q_\theta(s_i,a_i) - y_i)^2$
    \ENDFOR
\end{algorithmic}
\end{algorithm}
\end{minipage}
\end{small}
\end{figure*}

\textbf{Exact-FQI} (Algorithm~\ref{alg:fqiexact}): Exact-FQI computes the backup and projection on all state-action tuples without any sampling error.
It also assumes knowledge of dynamics and reward function to compute Bellman backups exactly. We use Exact-FQI to study convergence, distribution shift (by varying weighting distributions on transitions), and function approximation in the absence of sampling error. Exact-FQI eliminates errors due to sampling states, and computing inexact, sampled backups.

\textbf{Sampled-FQI} (Algorithm~\ref{alg:fqisampled}): Sampled-FQI is a special case of Exact-FQI,
where the Bellman error is approximated with Monte-Carlo estimates from a sampling distribution $\mu$, and the Bellman backup is approximated with samples from the dynamics as $r(s,a) + \gamma \max_{a'}Q(s', a')$. We use Sampled-FQI to study effects of overfitting. Sampled-FQI incorporates all sources of error -- arising from function approximation, sampling and also distribution shift.

\textbf{Replay-FQI} (Algorithm~\ref{alg:fqireplay}): Replay-FQI is a special case of Sampled-FQI that uses a \textit{replay buffer}~\citep{lin1992replay},
that saves past transition samples $(s, a, s', r)$, which are used for computing Bellman error. Replay-FQI strongle resembles DQN~\cite{Mnih2015}, lacking the online updates that allow $\mu$ to change within an FQI iteration. 
With large replay buffers, we expect the difference between Replay-FQI and DQN to be minimal as $\mu$ changes slowly.

We additionally investigate the following choices of weighting distributions ($\mu$) for the Bellman error. When sampling the Bellman error, these can be implemented by sampling directly from the distribution, or via importance sampling.

\textbf{Unif$(s,a)$}: Uniform weights over state-action space. This is the weighting distribution typically used by dynamic programming algorithms, such as FQI.

\textbf{$\pi(s,a)$}: The on-policy state-action marginal induced by $\pi$.

\textbf{$\pi^*(s,a)$}: The state-action marginal induced by $\pi^*$.

\textbf{Random$(s,a)$}: State-action marginal induced by executing uniformly random actions.

\textbf{Prioritized(s,a)}: Weights Bellman errors proportional to $|Q(s,a)-\backup Q(s,a)|$. This is similar to prioritized replay~\citep{Schaul2015} without importance sampling.

\textbf{Replay$(s,a)$} and \textbf{Replay10$(s,a)$}: Averaged state-action marginal of all policies (or the previous 10) produced during training. This simulates sampling uniformly from a replay buffer where infinite samples are collected from each policy. 

\subsection{Domains}

We evaluate our methods on suite of tabular environments where we can compute oracle values. This will help us compare, analyze and fix various sources of error by means of comparing the learned Q-functions to the true, oracle-compute Q-functions.
We selected 8 tabular domains, each with different qualitative attributes, including: gridworlds of varying sizes and observations, blind Cliffwalk~\citep{Schaul2015},
discretized Pendulum and Mountain Car based on implementations in OpenAI Gym~\citep{gym},
and a random sparsely connected graph. We give full details of these environments in Appendix~\ref{app:domains}, as well as their motivation for inclusion.

\subsection{Function Approximators}
Throughout our experiments, we use 2-layer ReLU networks, denoted by a tuple $(N, N)$ where N represents the number of units in a layer. The ``Tabular'' architecture refers to the case when no function approximation is used. 

\subsection{High-Dimensional Testing}

In addition to diagnostic experiments on tabular domains, we also wish to see if the observed trends hold true on high-dimensional environments. To this end, we include experiments on continuous control tasks in the OpenAI Gym benchmark~\citep{gym} (HalfCheetah-v2, Hopper-v2, Ant-v2, Walker2d-v2). In continuous domains, computing the maximum over actions of the Q-value is difficult ($\max_a Q(s,a)$). A common choice in this case is to use a second ``actor'' neural network to approximate $\arg\max_a Q(s,a)$~\cite{Lillicrap2015,pmlr-v80-fujimoto18a,Haarnoja18}. This approach most closely resembles Replay-FQI, but using the actor network in place of the max.

\section{Function Approximation and Convergence}
\label{sec:function_approx}

The first issue we investigate is the connection between function approximation and convergence properties.

\subsection{Technical Background}
As discussed in Section~\ref{sec:backrgound}, when function approximation is introduced to Q-learning, convergence guarantees are lost. This interaction between approximation and convergence has been a long-studied topic in reinforcement learning. In the control literature, it is closely related to the problems of state-aliasing or interference~\citep{Farrell95}. \citet{Baird1995} introduces a simple counterexample in which Watkin's Q-learning with linear approximators can cause unbounded divergence. In the policy evaluation scenario, \citet{Tsitsiklis1997} prove that on-policy TD-learning with linear function approximators converge, and methods such as GTD~\citep{Sutton09b} and ETD~\citep{Sutton2016} have extended results to off-policy cases.
In the control scenario, convergent algorithms such as SBEED~\citep{Dai2018} and Greedy-GQ~\citep{Maei2010} have been developed. However, several works have noted that divergence need not occur. \citet{munos2005erroravi} theoretically addresses the norm-mismatch problem, which show that unbounded divergence is impossible provided $\mu$ has adequate support and projections are non-expansive in p-norms. Concurrently to us, \citet{VanHesselt2018} experimentally find that unbounded divergence rarely occurs with DQN variants on Atari games. 
\subsection{How does function approximation affect convergence properties and suboptimality of solutions?}
The crucial quantities we wish to measure are a trend between function approximation and performance, and a measure for the bias in the learning procedure introduced by function approximation.
Thus, using Exact-FQI with uniform weighting (to remove sampling error), we measure the returns of the learning policy, and the $\linfnorm$ error between $Q^*$ and the solution found by Exact-FQI ($\lim_{t \to \infty} (\Projmu \backup)^t Q^0$) or the projection of the optimal solution  ($\Projmu Q^*$).
$\Projmu Q^*$ represents the best solution inside the model class, in absence of error from the bootstrapping process of FQI. Thus, the difference between FQI error and projection error represents the bias introduced by the bootstrapping procedure, while controlling for bias that is simply due to function approximation -- this quantity is roughly the \textit{inherent Bellman error} of the function class~\citep{munos2008finite}. This is the gap which can possibly be improved upon via better Q-learning algorithm design. We plot our results in Fig.~\ref{fig:function_approx}.

\begin{figure}
\caption{\label{fig:function_approx} Normalized returns and normalized Q-function error with function approximation, averaged across domains and seeds. We see that for small architectures, there is a significant gap between the solution found by FQI (FQI Error) and the best solution within the model class (Project Error).}
\includegraphics[width=0.95\columnwidth]{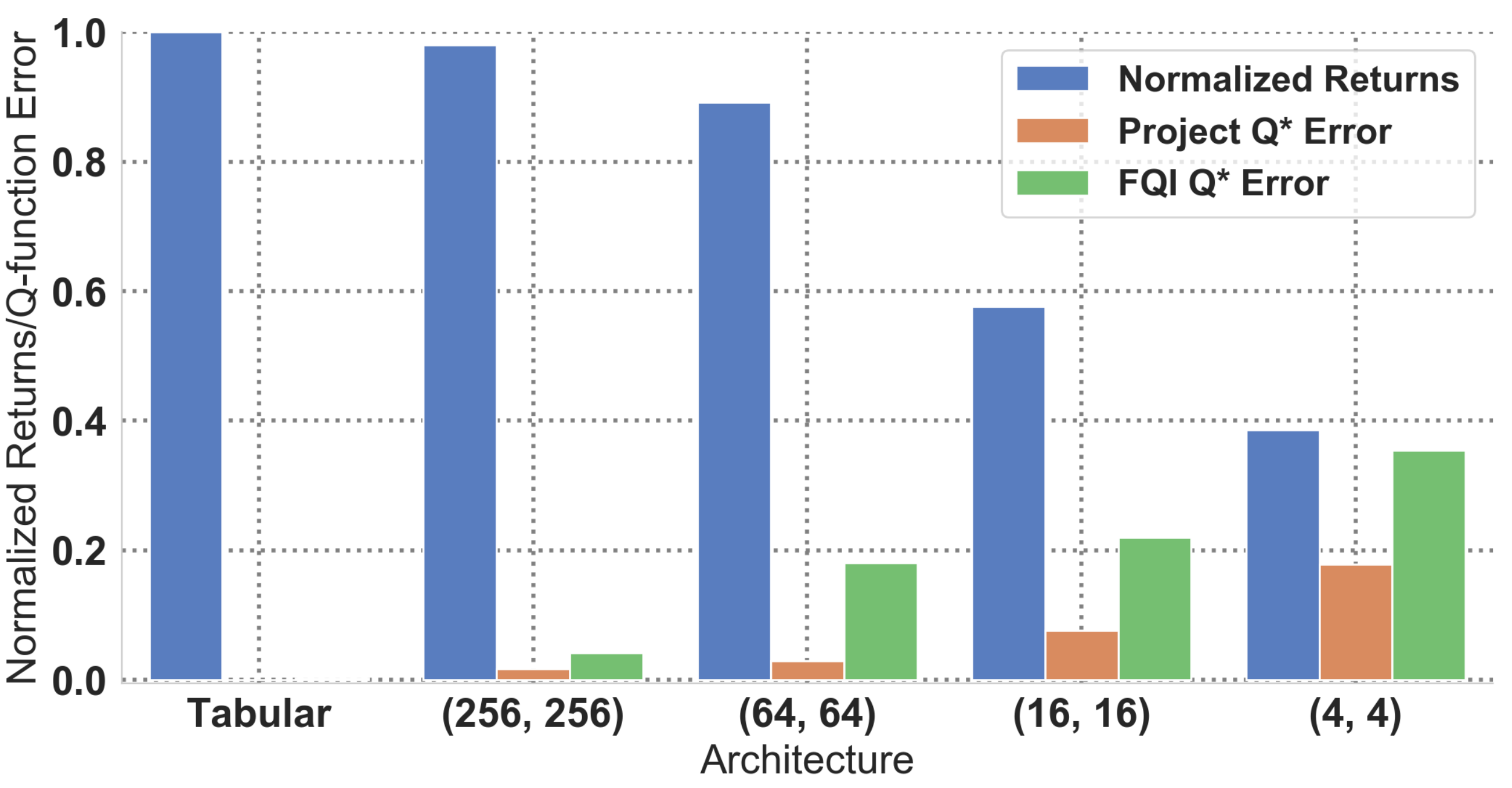}
\vspace{-0.2in}
\end{figure}

We first note the obvious trend that smaller architectures produce lower returns, and converge to more suboptimal solutions. 
However, we also find that smaller architectures introduce significant bias in the \textit{learning process}, and there is often a significant gap between the solution found by Exact-FQI and the best solution within the model class. 
This gap may be due to the fact that when the target is bootstrapped, we must be able to represent all Q-function along the path to the solution, and not just the final result~\citep{Bertsekas96}.
This observation implies that using large architectures is crucial not only because they have capacity to represent a better solution, but also because they are significantly easier to train using bootstrapping, and suffer less from nonconvergence issues. 
We also note that divergence rarely happens in practice. We observed divergence in 0.9\% of our experiments using function approximation, measured by the largest Q-value growing larger than 10 times that of $Q^*$. 

For high-dimensional problems, we present experiments on varying the architecture of the Q-network in SAC~\cite{Haarnoja18} in Appendix Fig.~\ref{fig:size_sac}. We still observe that large networks have the best performance, and that divergence rarely happens even in high-dimensional continuous spaces. We briefly discuss theoretical intuitions on apparent discrepancy between the lack of unbounded divergence in relation known counterexamples in Appendix~\ref{app:bounded_error}.

\section{Sampling Error and Overfitting}
\label{sec:overfitting}

A second source of error in minimizing the Bellman error, orthogonal to function approximation, is that of sampling or generalization error. The next issue we investigate is the effect of sampling error on Q-learning methods.

\subsection{Technical Background}
\begin{wrapfigure}{r}{0.6\columnwidth}
\centering
\vspace{-10pt}
\includegraphics[scale=0.27]{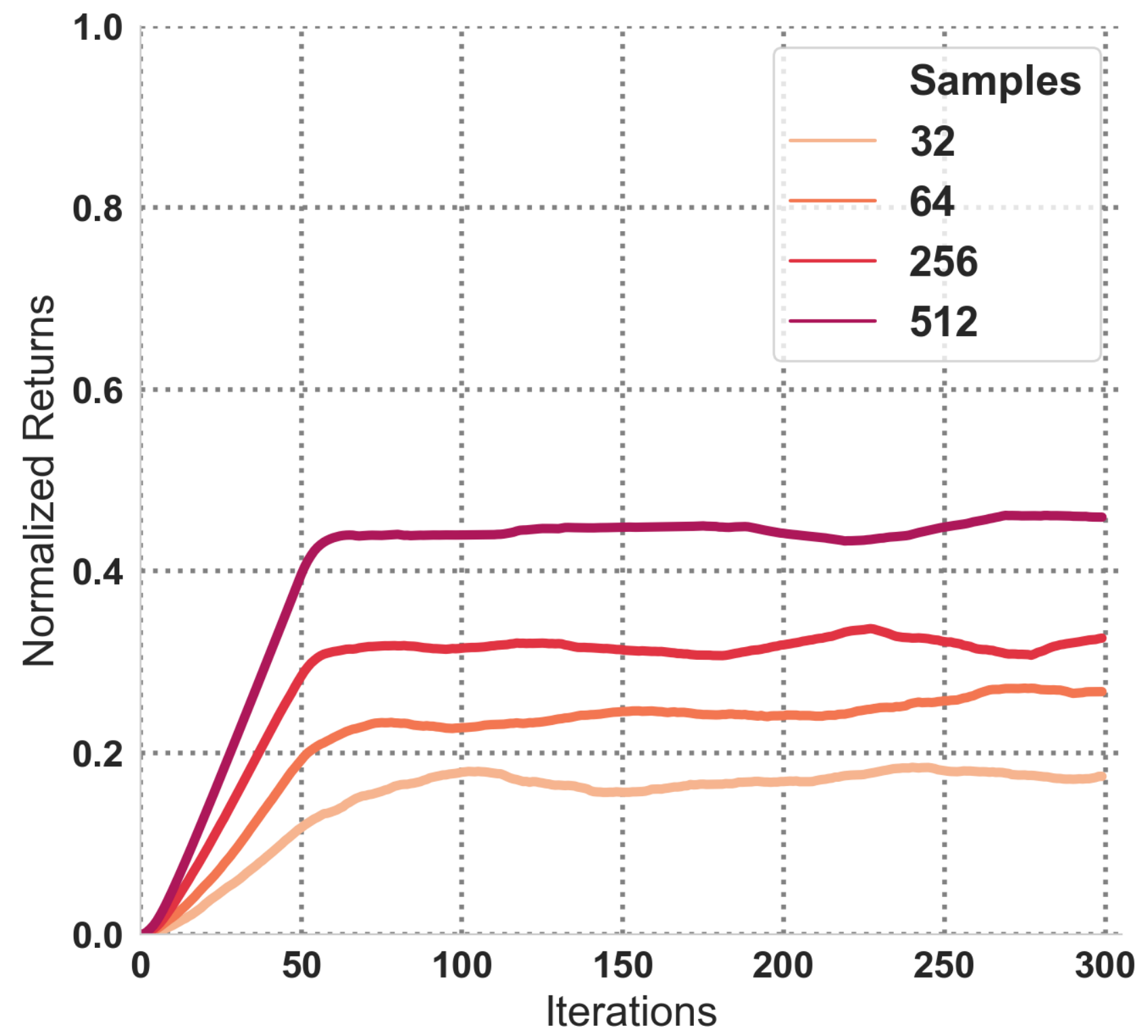}
\caption{\label{fig:sampling_256} Samples plotted with returns for a 256x256 network. More samples yields better performance.}
\vspace{-10pt}
\end{wrapfigure}

Approximate dynamic programming assumes that the projection of the Bellman backup (Eqn.~\ref{eqn:bellman_projection}) is computed exactly, but in reinforcement learning we can normally only compute the \textit{empirical Bellman error} over a finite set of samples. In the PAC framework, overfitting can be quantified by a bounded error in between the empirical and expected loss with high probability, which decays with sample size~\citep{Shalev2014}. \citet{munos2008finite, maillard2010finite, tosatto2017boosted} provide such PAC-bounds which account for sampling error in the context of Q-learning and value-based methods, and quantify the quality of the final solution in terms of sample complexity.

We analyze several key points that relate to sampling error. First, we show that Q-learning is prone to overfitting, and that this overfitting has a real impact on performance, in both tabular and high-dimensional settings. We also show that the replay buffer is in fact a very effective technique in addressing this issue, and discuss several methods to migitate the effects of overfitting in practice.

\begin{figure}[tb]
\includegraphics[width=0.95\columnwidth]{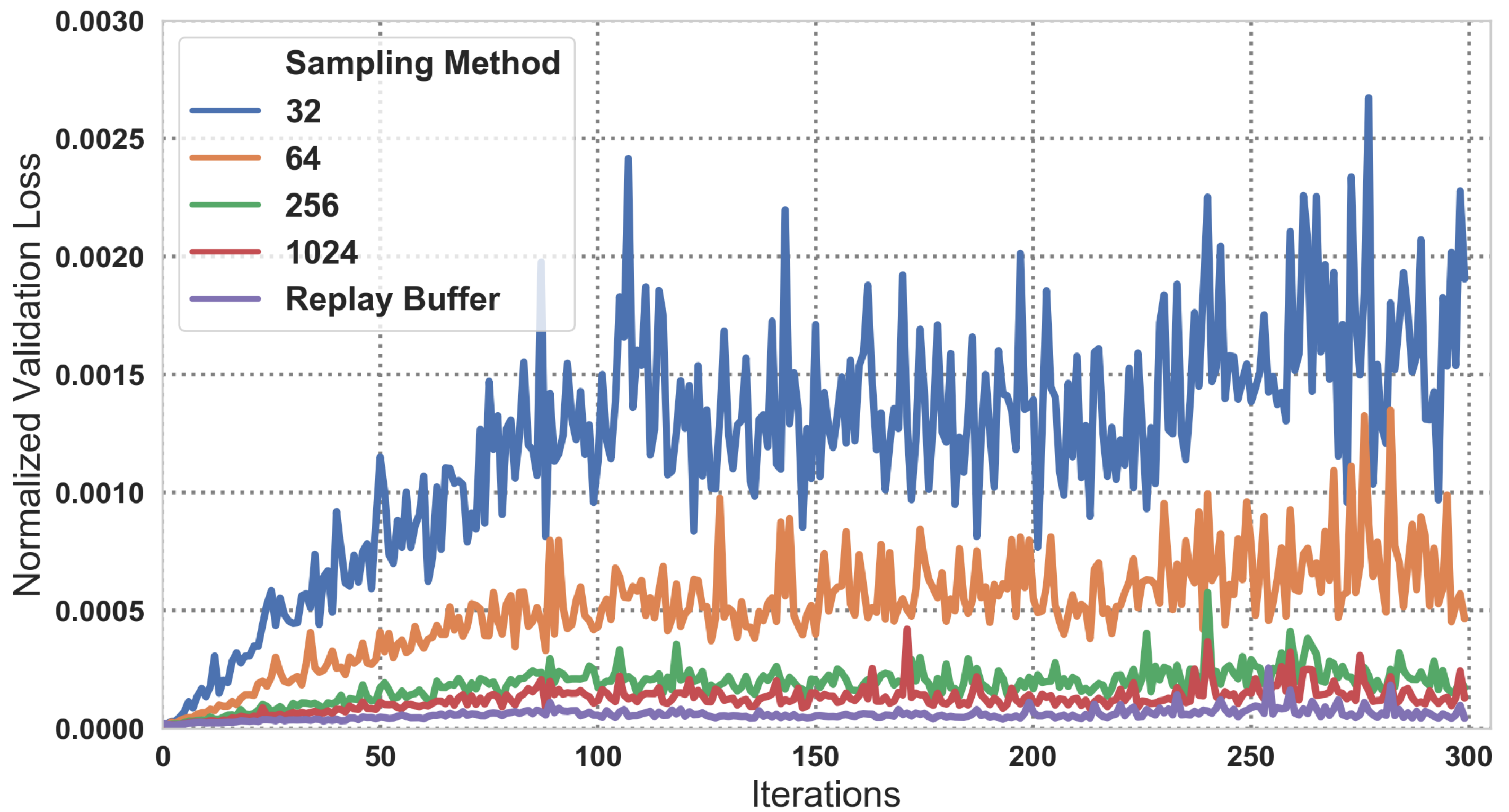}
\caption{\label{fig:sampling_validation_loss} On-policy validation losses for varying amounts of on-policy data (or replay buffer), averaged across environments and seeds. Note that sampling from the replay buffer has lower on-policy validation loss, despite bias from distribution shift.}
\end{figure}

\subsection{Quantifying Overfitting}
We first quantify the amount of overfitting that happens during training, by varying the number of samples. In order provide comparable validation errors across different experiments, we fix a reference sequence of Q-functions, $Q^1, ... , Q^N$, obtained during a normal training run. We then retrace the training sequence, and minimize the projection error $\Projmu(Q^t)$ for each training iteration, using varying amounts of on-policy data or sampling from a replay buffer. We measure the exact validation error (the expected Bellman error) at each iteration under the on-policy distribution, plotted in Fig.~\ref{fig:sampling_validation_loss}. We note the obvious trend that more samples leads to lower validation loss, confirming that overfitting can in fact occur. A more interesting observation is that sampling from the replay buffer results in the lowest on-policy validation loss, despite bias due to distribution mismatch from sampling off-policy data. As we discuss in Section~\ref{sec:analysis_nonstationarity}, we believe that replay buffers are mainly effective because they greatly reduce the effect of overfitting and create relatively good coverage over the state space, not necessarily due to reducing the effects of distribution shift.

Next, Fig.~\ref{fig:sampling_256} shows the relationship between number of samples and returns. We see a clear trend that higher sample count leads to improved learning speed and a better final solution, confirming our hypothesis that overfitting has a significant effect on the performance of Q-learning. A full sweep including architectures is presented in Appendix Fig.~\ref{fig:sampling_arch_sweep}. We observe that despite overfitting being an issue, larger architectures still perform better because the bias introduced by smaller architectures dominates.

\begin{figure}[tb]
\includegraphics[trim={0 0 7.0cm 0},clip,width=0.95\columnwidth]{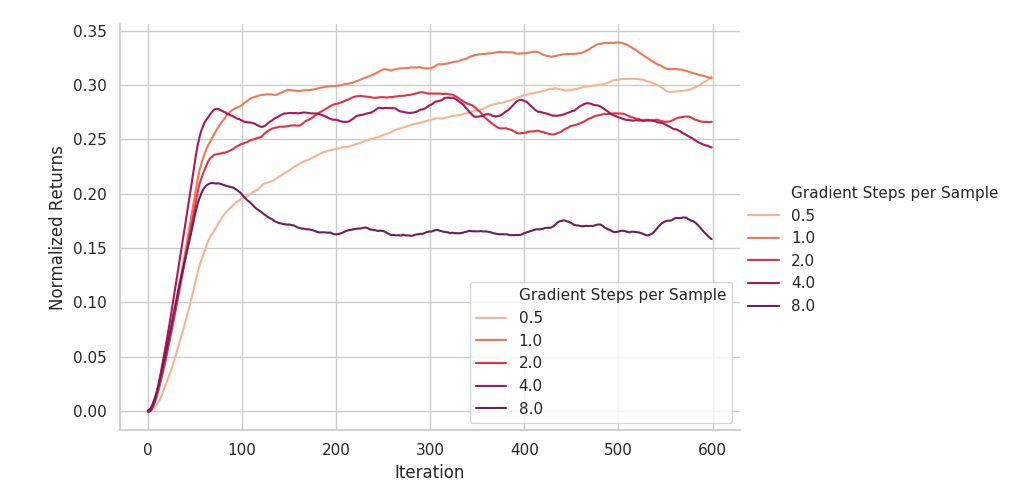}
\caption{\label{fig:fqi_grad_sweep}Normalized returns plotted over training iterations (32 samples are taken per iteration), for different ratios of gradient steps taken per sample during projection using Replay-FQI. We observe that intermediate values of gradient steps work best, and too many gradient steps hinders performance.}
\end{figure}

\begin{figure}[tb]
\includegraphics[trim={0 0 4.4cm 0},clip,width=0.95\columnwidth]{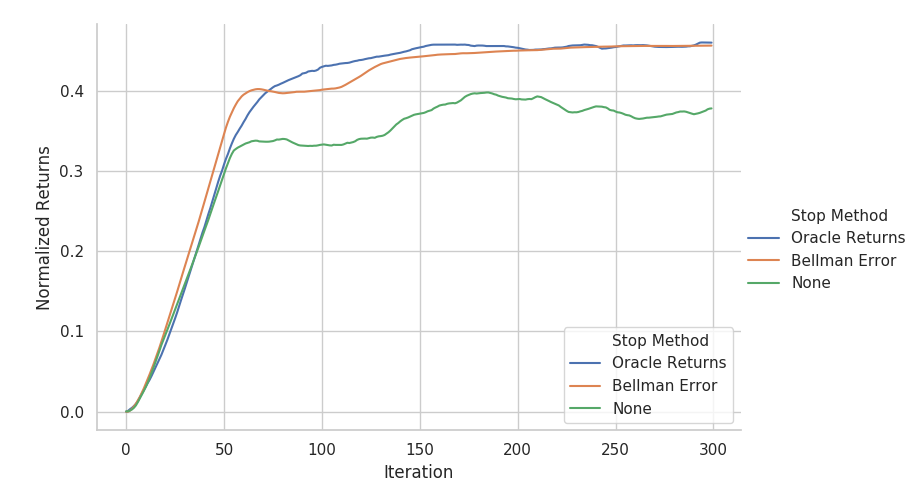}
\caption{\label{fig:validation_stop}Normalized returns plotted over training iterations (32 samples are taken per iteration), for different early stopping methods using Replay-FQI. We observe that using proper early stopping can result in a modest performance increase.}
\end{figure}

\subsection{What methods can be used to compensate for overfitting?}

Finally, we discuss methods to compensate for overfitting. One common method for reducing overfitting is to regularize the function approximator to reduce its capacity. However, as we have seen before that weaker architectures can give rise to suboptimal convergence, we instead study \textit{early stopping} methods to mitigate overfitting without reducing model size.
First, we observe that the number of gradient steps taken per sample in the projection step has an important effect on performance -- too few steps and the algorithm learns slowly, but too many steps and the algorithm may initially learn quickly but overfit. To show this, we run a hyperparameter sweep over the number of gradient steps taken per environment step in Replay-FQI and TD3 (TD3 uses 1 by default). Results for FQI are shown in Fig.~\ref{fig:fqi_grad_sweep}, and for TD3 in Appendix Fig.~\ref{fig:td3_grad_sweep}.

In order to understand whether better early stopping criteria can possibly help with overfitting, we employ \emph{oracle} early stopping rules. While neither of these rules can be used to solve overfitting in practice, these experiments can provide guidance for future methods and an ``upper bound'' on the best improvement that can be obtained from optimal stopping. We investigate two oracle early stopping criteria for setting the number of gradient steps: using the expected Bellman error and the expected returns of the greedy policy w.r.t. the current Q-function (oracle returns). We implement both methods by running the projection step of Replay-FQI to convergence using gradient descent, and afterwards selecting the intermediate Q-function which is judged best by the evaluation metric (lowest Bellman error or highest returns). Using such oracle stopping metrics results in a modest boost in performance in tabular domains (Fig.~\ref{fig:validation_stop}). Thus, we believe that there is promise in further improving such early-stopping methods for reducing overfitting in deep RL algorithms.

We might draw a few actionable conclusions from these experiments. First, overfitting is indeed a serious issue with Q-learning, and too many gradient steps or too few samples can lead to poor performance. Second, replay buffers and early stopping can be used to mitigate the effects of overfitting. Third, although overfitting is a problem, large architectures are still preferred, because the harm from function approximation bias outweighs the harm from increased overfitting with large models.
\section{Non-Stationarity}
\label{sec:analysis_nonstationarity}

In this section, we discuss issues related to the non-stationarity of the Q-learning process (relating to the Bellman backup and Bellman error minimization).

\subsection{Technical Background}
Instability in Q-learning methods is often attributed to the nonstationarity of the regression objective \citep{Lillicrap2015,Mnih2015}. 
Nonstationarity occurs in two places: in the changing target values $\backup Q$, and in a changing weighting distribution $\mu$ (``distribution shift'') (i.e., due to samples being taken from different policies). Note that a non-stationary objective, by itself, is not indicative of instability. For example, gradient descent can be viewed as successively minimizing linear approximations to a function: for gradient descent on $f$ with parameter $\theta$ and learning rate $\alpha$, we have the ``moving'' objective $\theta^{t+1} = \argmin{\theta}\{ \theta^T \nabla_\theta f(\theta^t) - \frac{1}{2\alpha} \normtt{\theta - \theta^t} \} = \theta^t - \alpha \nabla_\theta f(\theta^t)$. 
However, the fact that the Q-learning algorithm prescribes an update rule and not a stationary objective complicates analysis. Indeed, the motivation behind algorithms such as GTD~\citep{Sutton09a, Sutton09b} and residual methods~\citep{Baird1995,scherrer2010residual} can be seen as introducing a stationary objective that can be optimized with standard procedures such as gradient descent for increased stability.
Therefore, a key question to investigate is whether these non-stationarities are detrimental to the learning process.

\subsection{Does a moving target cause instability in the absence of a moving distribution?}

To study the moving target problem, we must first isolate the effects of a moving target, and study how the rate at which the target changes impacts performance. To control the rate at which the target changes, we introduce an additional smoothing parameter $\alpha$ to Q-iteration, where the target values are now computed as an $\alpha$-moving average over previous targets. We define the $\alpha$-smoothed Bellman backup, $\backup^\alpha$, as follows: 
\[ \backup^{\alpha}Q = \alpha \backup Q + (1-\alpha)Q\]
This scheme is inspired by the soft target update used in algorithms such as DDPG~\citep{Lillicrap2015} and SAC~\citep{Haarnoja2017} to improve the stability of learning. Standard Q-iteration uses a ``hard'' update where $\alpha=1$. A soft target update weakens the contraction of Q-iteration from $\gamma$ to $1-\alpha+\alpha\gamma$ (See Appendix~\ref{app:alpha_smoothed_q}),
so we expect slower convergence, but perhaps it is more stable under heavy function approximation error. We performed experiments with this modified backup using Exact-FQI under the $\text{Unif}(s,a)$ weighting distribution.

Our results are presented in Appendix Fig.~\ref{fig:smooth_fqi}.
We find that the most cases, the hard update with $\alpha=1$ results in the fastest convergence and highest asymptotic performance. However, for the smallest two architectures we used, $4 \times 4$ and $16 \times 16$, lower values of $\alpha$ (such as 0.1) achieve slightly higher asymptotic performance. Thus, while more expressive architectures are still stable under fast-changing targets, we believe that a slowly moving target may have benefits under heavy approximation error. This evidence points to either using large function approximators, in line with the conclusions drawn in the previous sections, or adaptively slowing the target updates when the architecture is weak (relative to the problem difficulty) and the projected Bellman error is therefore high.

\subsection{Does distribution shift impact performance?}
\label{sec:distr_shift}

\begin{figure}[tb]
\caption{\label{fig:distribution_shift_tv_loss} Distribution shift and loss shift plotted against time. Prioritized and on-policy distributions induce the greatest shift, whereas replay buffers greatly reduce the amount of shift.}
\includegraphics[width=0.51\columnwidth]{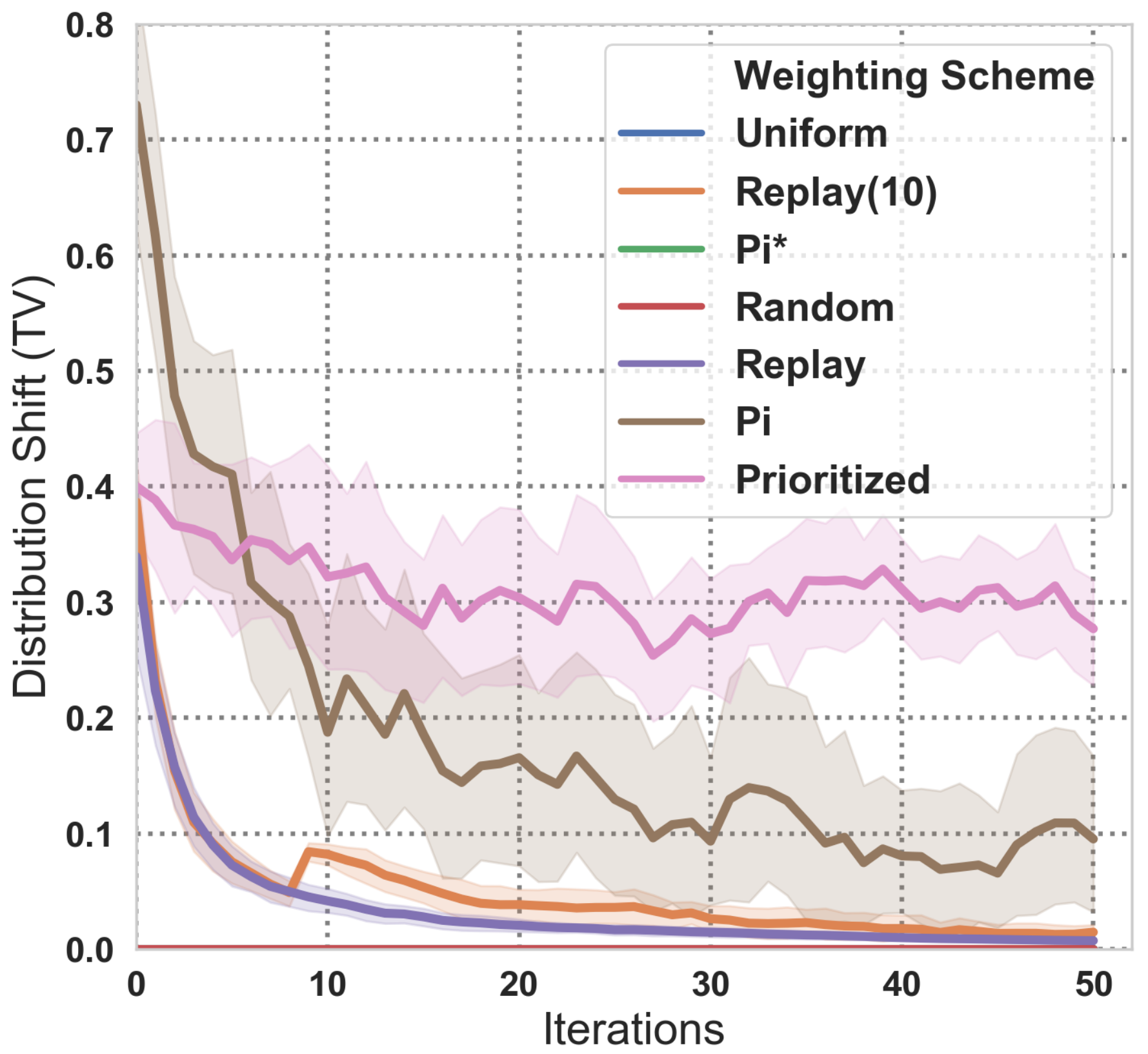}
\includegraphics[width=0.47\columnwidth]{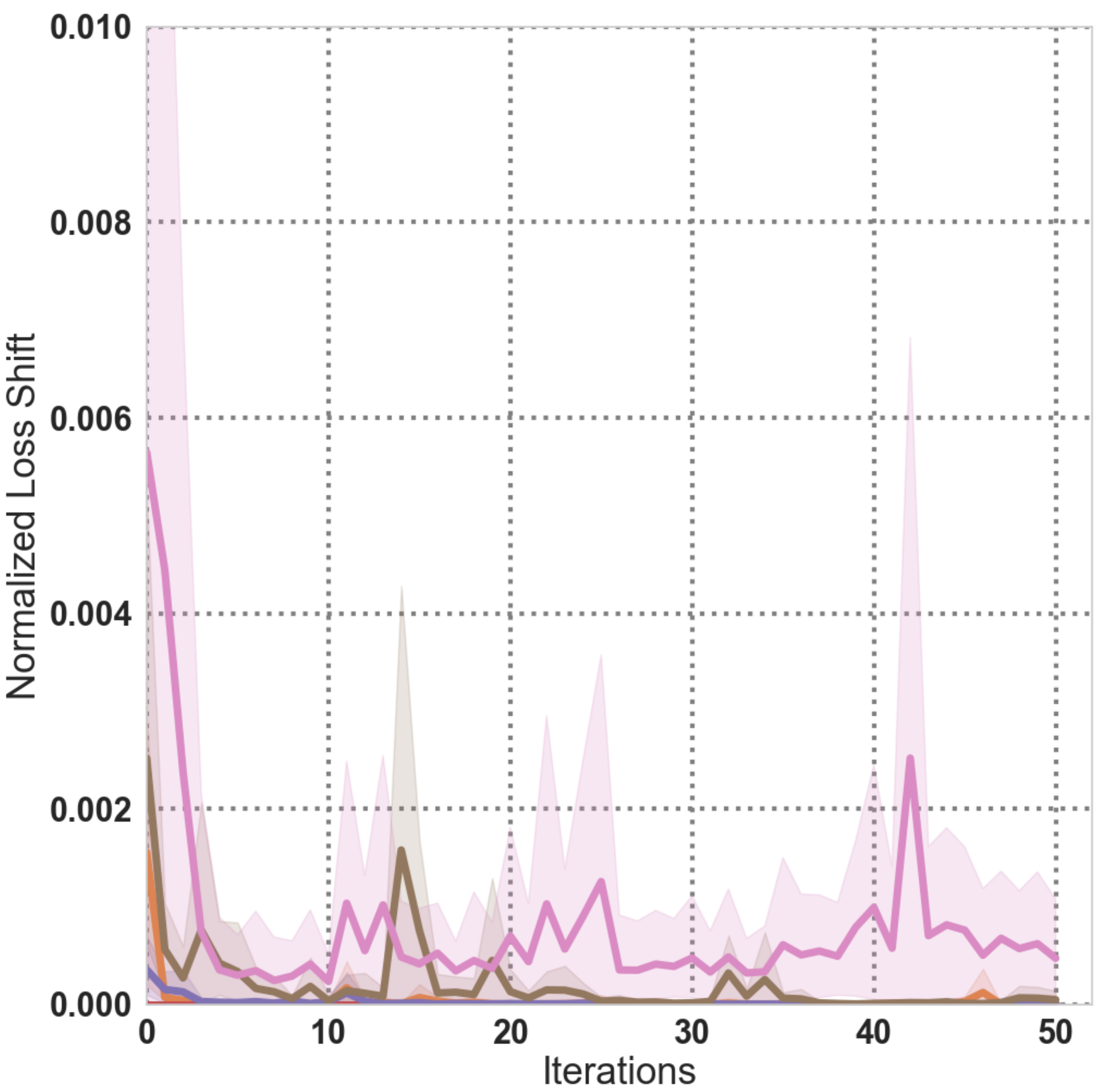}
\end{figure}

To study the distribution shift problem, we exactly compute the amount of distribution shift between iterations in total-variation distance, $D_{TV}(\mu^{t+1} || \mu^{t})$ and the ``loss shift'':
\[\mathbb{E}_{\mu^{t+1}}[ (Q^{t} - \backup Q^{t})^2] - \mathbb{E}_{\mu^{t}}[ (Q^{t} - \backup Q^{t})^2] .\] 
The loss shift quantifies the Bellman error objective when evaluated under a new distribution - if the distribution shifts to previously unseen or low support states, we would expect a highly inaccurate Q-value in such states, and a correspondingly high loss shift.

\begin{figure}[ht]
\caption{\label{fig:distribution_shift_vs_returns} Average distribution shift across time for different weighting distributions, plotted against returns for a 256x256 model. We find that distribution shift does not have strong correlation with returns.}
\includegraphics[width=0.99\columnwidth]{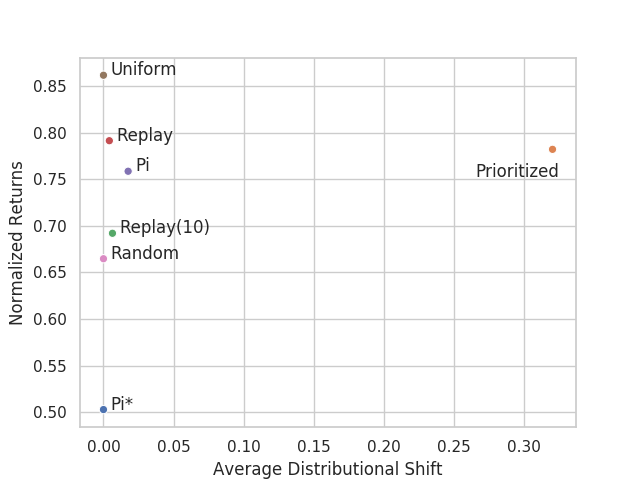}
\vspace{-0.2in}
\end{figure}

We run our experiments using Exact-FQI with a 256x256 layer architecture, and plot the distribution discrepancy and the loss discrepancy in Fig.~\ref{fig:distribution_shift_tv_loss}. 
We find that Prioritized$(s,a)$ has the greatest shift, followed by on-policy variants. Replay buffers greatly reduce distribution shift compared to on-policy learning, which is similar to the de-correlation argument cited for its use by~\citet{Mnih2015}.
However, we find that this metric correlates very little with the actual performance of the algorithm (Fig.~\ref{fig:distribution_shift_vs_returns}). For example, prioritized weighting performs well yet has high distribution shift.

Overall, our experiments indicate that nonstationarities in both distributions and target values, when isolated, do not cause significant stability issues. Instead, other factors such as sampling error and function approximation appear to have more significant effects on performance. In the light of these findings, we might therefore ask: can we design a \emph{better} sampling distribution, without regard for distributional shift and with regard for high-entropy, that results in better final performance, and is realizable in practice? We investigate this in the following section.

\section{Sampling Distributions}
\label{sec:sampling_distributions}

As alluded to in Section~\ref{sec:analysis_nonstationarity}, the choice of sampling distribution $\mu$ is an important design decision can have a large impact on performance. Indeed, it is not immediately clear which distribution is ideal for Q-learning. In this section, we hope to shed some light on this issue.

\subsection{Technical Background}
Off-policy data has been cited as one of the ``deadly triads'' for Q-learning~\citep{suttonrlbook}, which has potential to cause instabilities in learning. On-policy distributions~\citep{Tsitsiklis1997} and fixed behavior distributions~\citep{Sutton09a,Maei2010} have often been targeted for theoretical convergence analysis, and many works use importance sampling to correct for off-policyness~\citep{precup2001offpol, munos2016safe}
However, to our knowledge, there is relatively little guidance which compares how different weighting distributions compare in terms of convergence rate and final solutions.

Nevertheless, several works give hypotheses on good choices for weighting distributions.~\citep{munos2005erroravi} provides an error bound which suggests that ``more uniform'' weighting distributions can guarantee better worst-case performance.~\citep{NIPS2017_6913} suggests that when the state-distribution is fixed, the action distribution should be weighted by the optimal policy for residual Bellman errors. In deep RL, several methods have been developed to prevent instabilities in Q-Learning, such as prioritized replay~\citep{Schaul2015}, and mixing replay buffer with on-policy data~\citep{hausknecht2016policy,zhang2018deeper} have been found to be beneficial. In our experiments, we aim to empirically analyze multiple choices for weighting distributions to determine which are the most effective.

\subsection{What Are the Best Weighting Distributions in Absence of Sampling Error?}
\label{subsec:dist_shift_exact}
\begin{figure}[ht]
\caption{\label{fig:weighting_schemes} Weighting distribution versus architecture in Exact-FQI. Replay(s, a) consistently provides the highest performance. Note that Adversarial Feature Matching is comparable to Replay(s, a), but surprisingly better for small networks. }
\includegraphics[width=0.99\columnwidth]{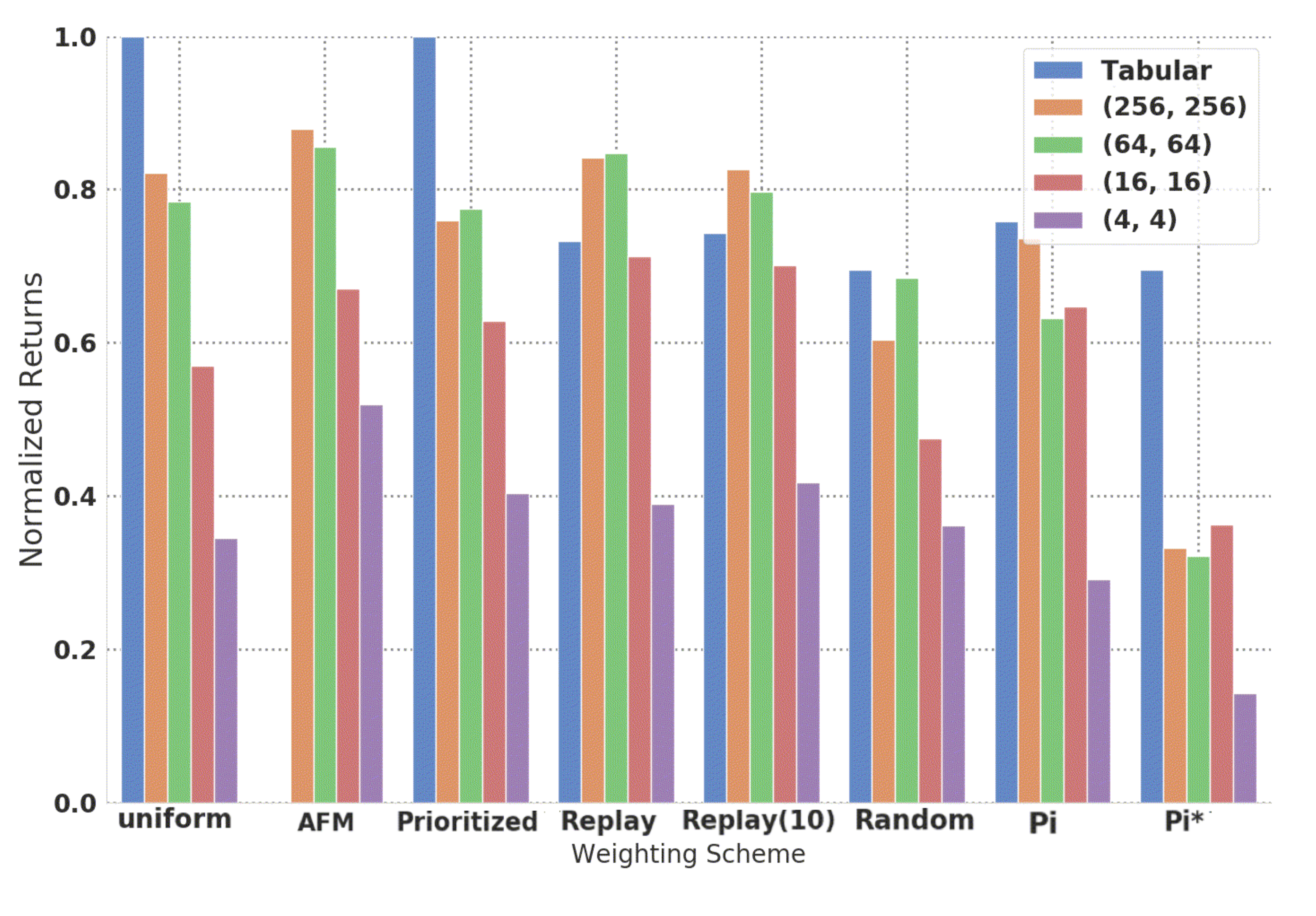}
\vspace{-0.3in}
\end{figure}

\begin{figure}[ht]
\caption{\label{fig:weighting_entropy_vs_returns} Normalized returns plotted against normalized entropy for different weighting distributions. All experiments use Exact-FQI with a 256x256 network. We see a general trend that high-entropy distributions lead to greater performance.}
\includegraphics[width=0.99\columnwidth]{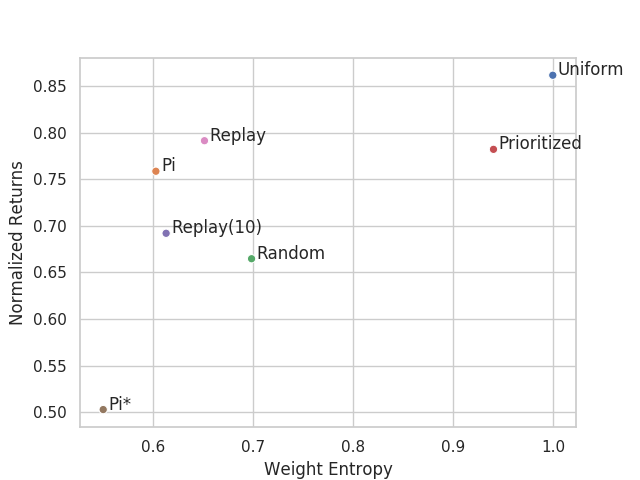}
\end{figure}

We begin by studying the effect of weighting distributions when disentangled from sampling error. We run Exact-FQI with varying choices of architectures and weighting distributions and report our results in Fig.~\ref{fig:weighting_schemes}. $\text{Unif}(s,a)$ $Replay(s,a)$, and $\text{Prioritized}(s,a)$ consistently result in the highest returns across all architectures.
We believe that these results are in favor of the \textit{uniformity} hypothesis: the top performing distributions spread weight across larger support of the state-action space. For example, a replay buffer contains state-action tuples from many policies, and therefore would be expected to have wider support than the state-action distribution of a single policy. We can see this general trend in Fig.~\ref{fig:weighting_entropy_vs_returns}. These distributions generally result in the tightest contraction rates, and allow the Q-function to focus on locations where the error is high. In the sampled setting, this observation motivates exploration algorithms that maximize state coverage (for example, ~\citet{hazan2018} solve an exploration objective which maximizes state-space entropy).
However, note that in this particular experiment, there is no sampling. All states are observed, just with different weights, thus isolating the issue of distributions from the issue of sampling.

\subsection{Designing a Better Off-Policy Distribution: Adversarial Feature Matching}
\label{sec:afm}

In our final study, we attempt to design a better weighting distribution using insights from previous sections that can be easily integrated into deep RL methods. We refer to this method as adversarial feature-matching (AFM). We draw upon three specific insights outlined in previous analysis. First, the function approximator should be incentivized to maximize its ability to distinguish states to minimize function approximation bias (Section~\ref{sec:function_approx}). Second, the weighting distribution should emphasize areas where the Q-function incurs high Bellman error, in order to minimize the discrepancy between $\ltwonorm$ norm error and $\linfnorm$ norm error. Third, more-uniform weighting distributions tend to be higher performant. The first insight was also demonstrated in \cite{martha2018sparse} where enforcing sparsity in the Q-function was found to provide locality in the Q-function which prevented catastrophic interference and provided better values for bootstrapping.

We propose to model our problem as a minimax game, where the weighting distribution is a parameterized adversary $p_\phi(s, a)$ which tries to \emph{maximize} the Bellman error, while the Q-function ($Q_\theta(s, a)$) tries to minimize it. 
Note that in the unconstrained setting, this game is equivalent to minimizing the $\linfnorm$ norm error in its dual-norm representation. However, in practical settings where minimizing stochastic approximations of the $\linfnorm$ norm can be difficult for neural networks (also noticed when using PER~\cite{VanHesselt2018}), it is crucial to introduce constraints to limit the power of the adversary. These constraints also make the adversary closer to the uniform distribution while still allowing it to be sufficiently different at specific state-action pairs.

We elect to use a feature matching constraint which enforces the expected feature vectors, $\mathbb{E}[\Phi(s)]$, under $p_\phi(s, a)$ to \emph{roughly} match the expected feature vector under uniform sampling from the replay buffer. We can express the output of a neural network Q-function as $Q_\theta(s, a) = w_{a}^T \Phi_\theta(s)$ or, in the continuous case, as $Q_\theta(s, a) = w^T \Phi_\theta(s, a)$,  where the feature vector $\Phi_\theta(s), \Phi_\theta(s, a)$ represent the the output of all but the final layer.
Intuitively, this constraint restricts the adversarial sampler to distributing probability mass among states (or state-action pairs) that are perceptually similar to the Q-function, which in turn forces the Q-function to reduce state-aliasing by learning features that are more separable. Note that, in our case, $\phi(s,a) = \nabla_{w} Q_{w, \theta}(s, a)$. This also provides a natural extension of our method by performing expected gradient matching over all parameters ($\mathbb{E}_{p_\phi(s, a)}[\nabla_{w, \theta} Q_{w, \theta}(s, a)]$), instead of matching only $\Phi$ (we leave it to future work to explore this direction). Formally, this objective is given as follows:
\begin{multline*}
    \min_{\theta, w} \max_{\phi} \mathbb{E}_{p_\phi(s, a)} [(Q_{w, \theta} (s, a) - y(s, a))^2]\\
s.t.~~ \vert\vert \mathbb{E}_{p_\phi(s, a)}[\Phi(s)] - \frac{\sum_i \Phi(s_i)}{N} \vert\vert \leq \varepsilon
\end{multline*}
Note that $\Phi(s)$ is a function of $\theta$ but, while solving the maximization, $\theta$ is assumed to be a constant. This is equivalent to solving only the inner maximization with a constraint, and empirically provides better stability. Implementation details for AFM are provided in \textbf{Appendix \ref{app:adversarial}}. The $\frac{\sum_{i} \Phi(s_i)}{N}$ denotes an estimator for the true expectation under some sampling distribution, such as a uniform distribution over all states and actions (in exact FQI) or the replay buffer distribution. So, $\frac{\sum_{i} \Phi(s_i)}{N} \approx E_{p_{rb}}[\Phi]$ holds when using a replay buffer.

While both AFM and PER tend to upweight samples in the buffer with a high Bellman error, PER explicitly attempts to \emph{reduce} distribution shift via importance sampling. As we observed in Section~\ref{sec:sampling_distributions}, distributional shift is not actually harmful in practice, and AFM dispenses with this goal, instead explicitly aiming to rebalance the buffer to attain better coverage via adversarial optimization. In our experiments, this results in substantially better performance, consistent with the hypothesis that coverage, rather than reduction of distributional shift, is the most important property in a sampling distribution.

In tabular domains with Exact-FQI, we find that AFM performs at par with the top performing weighting distributions, such as $\text{Unif}(s, a)$ and \textit{better} than $\text{Prioritized}(s, a)$ (Fig.~\ref{fig:weighting_schemes}). This confirms that adaptive prioritization works better than Prioritized($s, a$). Another benefit of AFM is its robustness to function approximation and the performance gains in the case of small architectures (say, $(4, 4)$) are particularly noticeable. (Fig.~\ref{fig:weighting_schemes})

In tabular domains with Replay-FQI (Table~\ref{table:final}), we also compare AFM to prioritized replay (PER)~\citep{Schaul2015}, where AFM and PER perform similarly in terms of normalized returns. Note that AFM reweights samples drawn uniformly from the buffer, whereas PER changes which samples are actually drawn. We also evaluate a variant of AFM (AFM+Sampling in Table~\ref{table:final}) which changes which samples instead of reweighting.
Essentially, in this version we sample from the replay buffer using probabilities determined by the AFM optimization, rather than using importance sampling while making bellman updates. We note that, in Table~\ref{table:final}, AFM+Sampling performs strictly better than AFM and PER.

We further evaluate AFM on MuJoCo tasks with the TD3 algorithm~\citep{pmlr-v80-fujimoto18a} and the entropy constrained SAC algorithm~\citep{Haarnoja18}. We find that in all 3 tested domains (Half-Cheetah, Hopper and Ant), AFM yields substantial empirical improvement in the case of TD3 (Fig.~\ref{fig:td3_results_adv}) and performs slightly better than entropy constrained SAC (Fig.~\ref{fig:sac_results_adv}). Surprisingly, we found PER to not work very well in these domains. In light of these results, we conclude that: (1) the choice of sampling distribution is very  important for performance, and (2) considerations such as incorporating knowledge about the function approximator (for example, through $\Phi$) into the choice of $\mu$ (the sampling/weighting distribution) can be very effective.

\begin{figure*}[t]
\begin{subfigure}[t]{0.33\textwidth}
    \includegraphics[scale=0.24]{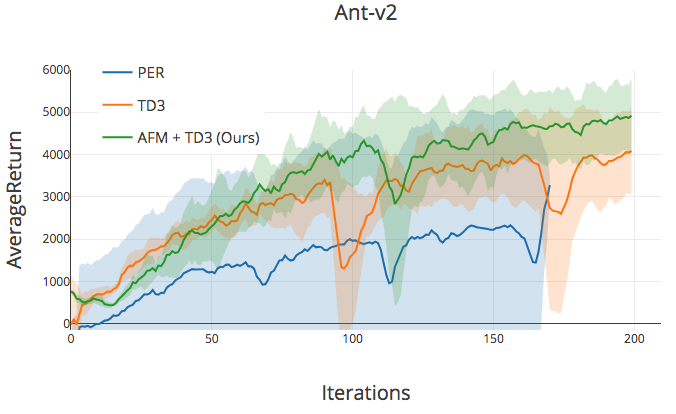}
    \caption{Ant-v2}
\end{subfigure}
\begin{subfigure}[t]{0.33\textwidth}
    \includegraphics[scale=0.24]{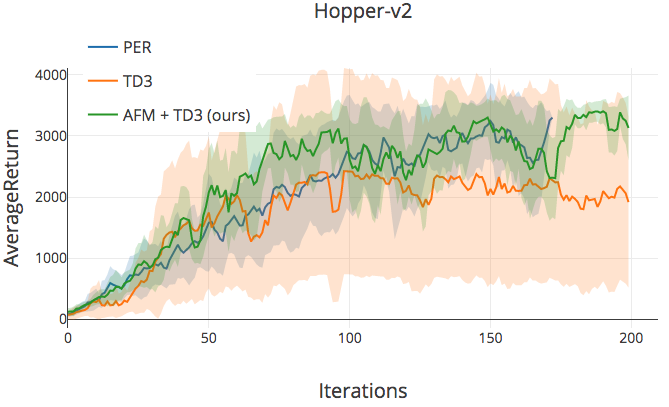}
    \caption{Hopper-v2}
\end{subfigure}
\begin{subfigure}[t]{0.33\textwidth}
    \includegraphics[scale=0.24]{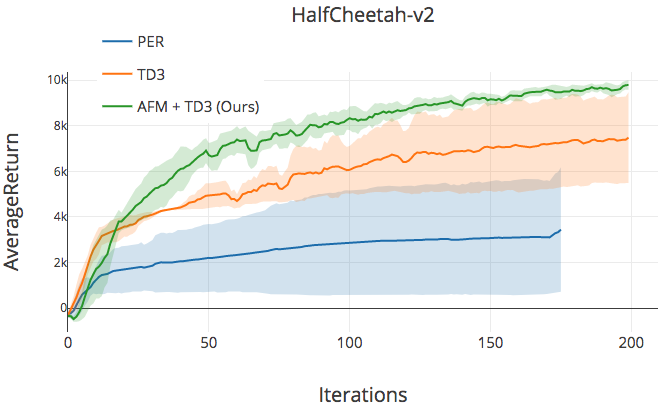}
    \caption{HalfCheetah-v2}
\end{subfigure}
\caption{Average Return for rollouts performed with a trained the TD3 algorithm with/without \textbf{AFM} (Ours) and with Prioritized Replay (PER). Note that on an average AFM performs better than the baseline and the Prioritized Replay. Each iteration on the x-axis corresponds to 5000 environment steps.}
\label{fig:td3_results_adv}
\end{figure*}

\begin{figure*}[t]
\begin{subfigure}[t]{0.33\textwidth}
    \includegraphics[scale=0.24]{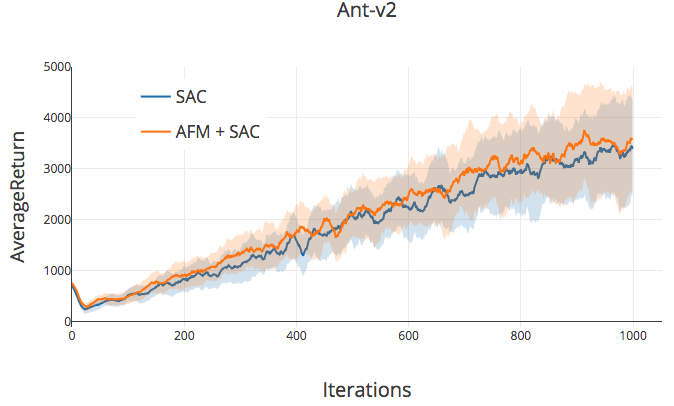}
    \caption{Ant-v2}
\end{subfigure}
\begin{subfigure}[t]{0.33\textwidth}
    \includegraphics[scale=0.24]{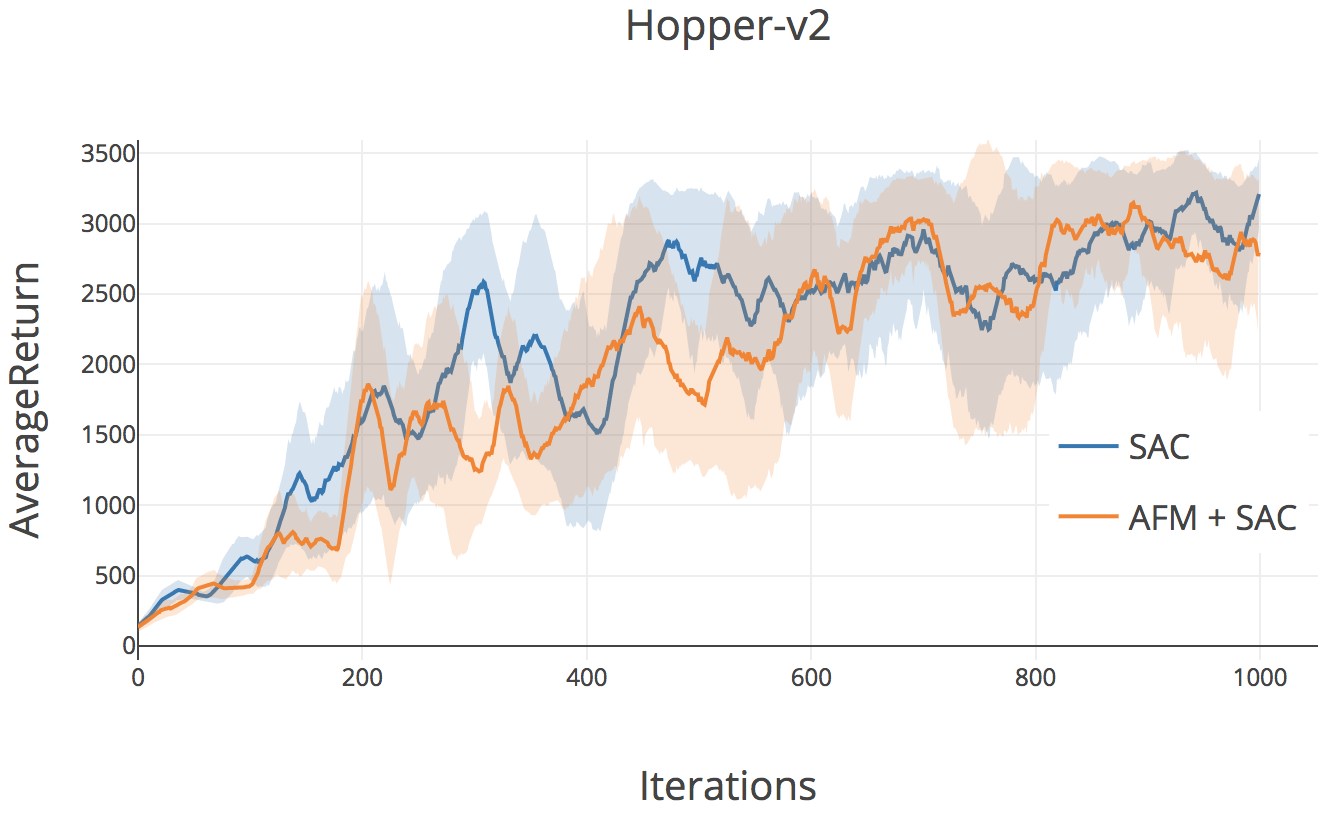}
    \caption{Hopper-v2}
\end{subfigure}
\begin{subfigure}[t]{0.33\textwidth}
    \includegraphics[scale=0.24]{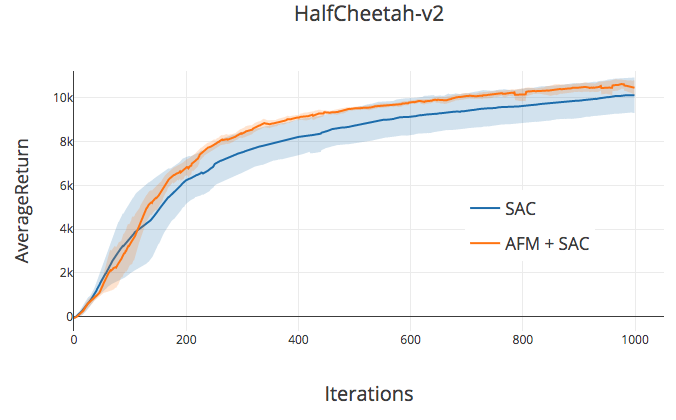}
    \caption{HalfCheetah-v2}
\end{subfigure}
\caption{Average Return for rollouts performed with a trained SAC model with temperature auto-tuning~\citep{haarnoja2018sacapps} with/without AFM. Note that on an average AFM performs slightly better and is always atleast at par with SAC. Each iteration on the x-axis corresponds to 1000 environment steps.}
\label{fig:sac_results_adv}
\end{figure*}

\begin{table}
    \centering
    \small{
    \begin{tabular}{|c|r|r|}
    \hline
    \textbf{Sampling distribution} & \textbf{Norm. Returns} & \textbf{Norm. Returns} \\
     & \textbf{(16, 16)} & \textbf{(64, 64)} \\
     \hline\hline
    None &  0.18 & 0.23 \\
    \hline
    Uniform(s, a) &  0.19 & 0.25 \\
    \hline
    $\pi(s, a)$ &  \textbf{0.45} & 0.39 \\
    \hline
    $\pi^*(s, a)$ & 0.30 & 0.21 \\
    \hline
    Prioritized(s, a) & 0.17 & 0.33 \\
    \hline
    PER~\cite{Schaul2015} & 0.42 & \textbf{0.49}\\
    \hline
    \textbf{AFM} (Ours) & 0.41 & \textbf{0.48} \\
    \hline
    \textbf{AFM + Sampling} (Ours) & 0.43 & \textbf{0.51} \\
    \hline
     \end{tabular}}
    \caption{\label{table:final}Average Performance of various sampling distributions for (16, 16) and (64, 64) neural nets in the setting with replay buffers where sampling errors, function approximation error and distribution shift, all of them coexist, averaged across 5 random seeds. PER, our AFM and on-policy sampling perform roughly at par on benchmark tasks in expectation when using (16, 16) architectures. However, note that  $\pi(s,a)$ is generally computationally intractable. Another point to note is that AFM performs as good as PER just by virtue of weighting and not \textit{sampling}. AFM+Sampling which is the sampling analogue of AFM works better than PER and AFM on average on tabular domains.}
    \vspace{-0.2in}
\end{table}

\section{Conclusions and Discussion}

From our analysis, we have several broad takeaways for the design of deep Q-learning algorithms. 

\textbf{Potential convergence issues} with Q-learning do not seem to be endemic empirically, but function approximation still has a strong impact on the solution to which these methods converge. This impact goes beyond just approximation error, suggesting that Q-learning methods do find suboptimal solutions (within the given function class) with smaller function approximators. However, expressive architectures largely mitigate this problem, suffer less from bootstrapping error, converge faster, and more stable with moving targets.

\textbf{Sampling error} can cause substantial overfitting problems with Q-learning. However, replay buffers and early stopping can mitigate this problem, and the biases incurred from small function approximators outweigh any benefits they may have in terms of overfitting. We believe the best strategy is to keep large architectures but carefully select the number of gradient steps used per sample. We showed that employing oracle early stopping techniques can provide huge benefits in the performance in Q-learning. This motivates the future research direction of devising early stopping techniques to dynamically control the number of gradient steps in Q-learning, rather than setting it as a hyperparameter as this can give rise to big difference in performance.

\textbf{The choice of sampling or weighting distribution} has significant effect on solution quality, even in the absence of sampling error. Surprisingly, we do not find on-policy distributions to be the most performant, but rather methods which have high state-entropy and spread mass uniformly among state-action pairs, seem to be highly effective for training. Based on these insights, we propose a new weighting distribution which balances high-entropy and state aliasing, AFM, that yields fair improvements in both tabular and continuous domains with state-of-the-art off-policy RL algorithms.

Finally, we note that there are many other topics in Q-learning that we did not investigate, such as overestimation bias and multi-step returns. We believe that these issues too could be studied in future work with our oracle-based analysis framework.

\section*{Acknowledgements}
We thank Vitchyr Pong and Kristian Hartikainen for providing us with implementations of RL algorithms. We thank Chelsea Finn for comments on an earlier draft of this paper. SL thanks George Tucker for helpful discussion. We thank Google, NVIDIA, and Amazon
for providing computational resources. This research was supported by Berkeley DeepDrive, NSF IIS-1651843 and IIS-1614653, the DARPA Assured Autonomy program, and ARL DCIST CRA W911NF-17-2-0181.



\bibliography{example_paper}
\bibliographystyle{icml2019}

\clearpage
\appendix
\section*{Appendices}

\section{Benchmark Tabular Domains}
\label{app:domains}

We evaluate on a benchmark of 8 tabular domains, selected for qualitative differences.

\textbf{4 Gridworlds}. The Gridworld environment is an NxN grid with randomly placed walls. The reward is proportional to Manhattan distance to a goal state (1 at the goal, 0 at the initial position), and there is a 5\% chance the agent travels in a different direction than commanded. We vary two parameters: the size ($16 \times 16$ and $64 \times 64$), and the state representations. We use a ``one-hot'' representation, an (X, Y) coordinate tuple (represented as two one-hot vectors), and a ``random'' representation, a vector drawn from $\mathcal{N}(0, 1)^N$, where N is the width or height of the Gridworld. The random observation significantly increases the challenge of function approximation, as significant state aliasing occurs.

\textbf{Cliffwalk}: Cliffwalk is a toy example from \citet{Schaul2015}. It consists of a sequence of states, where each state has two allowed actions: advance to the next state or return to the initial state. A reward of 1.0 is obtained when the agent reaches the final state. Observations consist of vectors drawn from $\mathcal{N}(0, 1)^{16}$.

\textbf{InvertedPendulum and MountainCar}: InvertedPendulum and MountainCar are discretized versions of continuous control tasks found in OpenAI gym~\citep{gym}, and are based on problems from classical RL literature. 
In the InvertedPendulum task, an agent must swing up an pendulum and hold it in its upright position. The state consists of the angle and angular velocity of the pendulum. Maximum reward is given when the pendulum is upright. The observation consists of the $\sin$ and $\cos$ of the pendulum angle, and the angular velocity.
In the MountainCar task, the agent must push a vehicle up a hill, but the hill is steep enough that the agent must gather momentum by swinging back and forth within a valley in order to reach the top. The state consists of the position and velocity of the vehicle.

\textbf{SparseGraph}: The SparseGraph environment is a 256-state graph with randomly drawn edges. Each state has two edges, each corresponding to an action. One state is chosen as the goal state, where the agent receives a reward of one.

\section{Fitted Q-iteration with Bounded Projection Error}
\label{app:bounded_error}

When function approximation is introduced to Q-iteration, we lose guarantees that our solution will converge to the optimal solution $Q^*$, because the composition of projection and backup is no longer guaranteed to be a contraction under any norm. However, this does not imply divergence, and in most cases it merely degrades the quality of solution found.

This can be seen by recalling the following result from~\citep{Bertsekas96}, that describes the quality of the solution obtained by fitted Q-iteration (FQI) when the projection error at each step is bounded. The conclusion is that FQI converges to an $\linfnorm$ ball around the optimal solution which scales proportionally with the projection error. While this statement does not claim that divergence cannot occur in general (this theorem can only be applied in retrospect, since we cannot always uniformly bound the projection error at each iteration), it nevertheless offers important intuitions on the behavior of FQI under approximation error. For similar results concerning $\mu$-weighted $\ltwonorm$ norms, see~\citep{munos2005erroravi}.

\begin{theorem}[Bounded error in fitted Q-iteration]
\label{thm:bounded_qi_bound}
Let the projection or Bellman error at each iteration of FQI be uniformly bounded by $\delta$, i.e. $\norm{\hat{Q}_{i+1} - \backup \hat{Q}_i}_\infty \le \delta\ \forall\ i$. Then, the error in the final solution is bounded as

\[  \lim_{i \to \infty} \norm{\hat{Q}_i - Q^*}_\infty \le \frac{\delta}{1-\gamma} \]
\end{theorem}
\begin{proof}
See of Chapter 6 of~\citet{Bertsekas96}.
\end{proof}

We can use this statement to provide a bound on the performance of the final policy. 

\begin{corollary}
Suppose we run fitted Q-iteration, and let the projection error at each iteration be uniformly bounded by $\delta$, i.e. $\norm{\hat{Q}_{i+1} - \backup \hat{Q}_i}_\infty \le \delta\ \forall\ i$. Letting $\eta(\pi)$ denote the returns of a policy $\pi$, the the performance of the final policy is bounded as:

\[  \lim_{i \to \infty} |\eta(\pi^i) - \eta(\pi^*)| \le \frac{2\gamma\delta}{(1-\gamma)^2} \]
\end{corollary}
\begin{proof}
This result is obtained by substituting Thm.~\ref{thm:bounded_qi_bound} into Propositon 6.1 of ~\citet{Bertsekas96}.
\end{proof}

\subsection{Unbounded divergence in FQI}
\label{app:unbounded_divergence}
Because $\ltwonorm$ norms are bounded by the $\linfnorm$ norm, Thm.~\ref{thm:bounded_qi_bound} implies that \textit{unbounded} divergence is impossible when weighting distribution has positive support at all states and actions (i.e. $\mu(s,a) > 0\ \forall\ (s, a) \in (\mathcal{S}, \mathcal{A})$), and the projection is non-expansive in the $\ltwonorm$ norm (such as when using linear approximators). 

We can bound the $\mu$-weighted $\ltwonorm$ in terms of the $\linfnorm$ as follows: $\normtmu{\cdot} \le \frac{1}{\min \mu(s,a)} \normt{\cdot} \le \frac{\sqrt{|\mathcal{S}||\mathcal{A}|}}{\min \mu(s,a)} \norm{\cdot}_\infty$. Thus, we can apply Thm.~\ref{thm:bounded_qi_bound} with $\delta = \frac{\sqrt{|\mathcal{S}||\mathcal{A}|}}{\min \mu(s,a)}$ to show that unbounded divergence is impossible. Note that because this bound scales with the size of the state and action spaces, it is fairly loose in many practical cases, and practitioners may nevertheless see Q-values grow to large values (tighter bounds concerning L2 norms can be found in \cite{munos2005erroravi}, which depend on the transition distribution). It also suggests that distributions which are fairly uniform (so as to maximize the denominator) can perform well.

When the weighting distribution $\mu$ does not have support over all states and actions, divergence can still occur, as noted in the counterexamples such as Section 11.2 of ~\citet{suttonrlbook}. In this case, we consider two states (state 1 and 2) with feature vectors 1 and 2, respectively, and a linear approximator with parameter $w$. There exists a single action with a deterministic transition from state 1 to state 2, and we only sample the transition from state 1 to state 2 (i.e. $\mu(s,a)$ is 1 for state 1 and 0 for state 2). All rewards are 0. In this case, the projected Bellman backup takes the form:
\[w^{t+1} = \argmin{w} (w - 2\gamma w^t)^2\]
Which will cause unbounded growth $\lim_{t \to \infty} w_t = \infty$ when iterated, provided $\gamma > 0.5$. However, if we add a transition from state 2 back to itself or to state 1, and place nonzero probability on sampling these transitions, divergence can be avoided.

\section{$\alpha$-smoothed Q-iteration}
\label{app:alpha_smoothed_q}

In this section we show that the $\alpha$-smoothed Bellman backup introduced in Section~\ref{sec:analysis_nonstationarity} is still a valid Q-iteration method, in that it is a contraction (for $1 \ge \alpha > 0$) and thus converges to $Q^*$.

We define the $\alpha$-smoothed Bellman backup as:
\[ \backup^{\alpha}Q = \alpha \backup Q + (1-\alpha)Q\]

\begin{theorem}[Contraction rate of the $\alpha$-smoothed Bellman backup]
\label{thm:alpha_smoothed_q}
$\backup^{\alpha}$ is a $1-\alpha+\gamma\alpha$-contraction:
\[ \norm{\backup^{\alpha}Q_1 - \backup^{\alpha}Q_2}_\infty \le (1-\alpha+\gamma\alpha)\norm{Q_1 - Q_2}_\infty\]
\end{theorem}
\begin{proof}
This statement follows from straightforward application of the triangle rule and the fact that $\backup$ is a $\gamma$-contraction:
\begin{align*}
&\norm{\backup^{\alpha}Q_1 - \backup^{\alpha}Q_2}_\infty \\
&= \norm{(\alpha \backup Q_1 - (1-\alpha)Q_1) - (\alpha \backup Q_2 - (1-\alpha)Q_2 }_\infty \\
&= \norm{\alpha (\backup Q_1 - \backup Q_2) + (1-\alpha)(Q_1-Q_2) }_\infty \\
&\le \alpha \norm{\backup Q_1 - \backup Q_2}_\infty + (1-\alpha)\norm{Q_1-Q_2 }_\infty \\
&\le \alpha \gamma \norm{Q_1 - Q_2}_\infty + (1-\alpha)\norm{Q_1-Q_2 }_\infty \\
&=  (1-\alpha + \alpha\gamma) \norm{Q_1 - Q_2}_\infty
\end{align*}
\end{proof}

\begin{figure*}[t!]
\begin{subfigure}[t]{1.0\textwidth}
    \centering
    \includegraphics[width=0.99\textwidth]{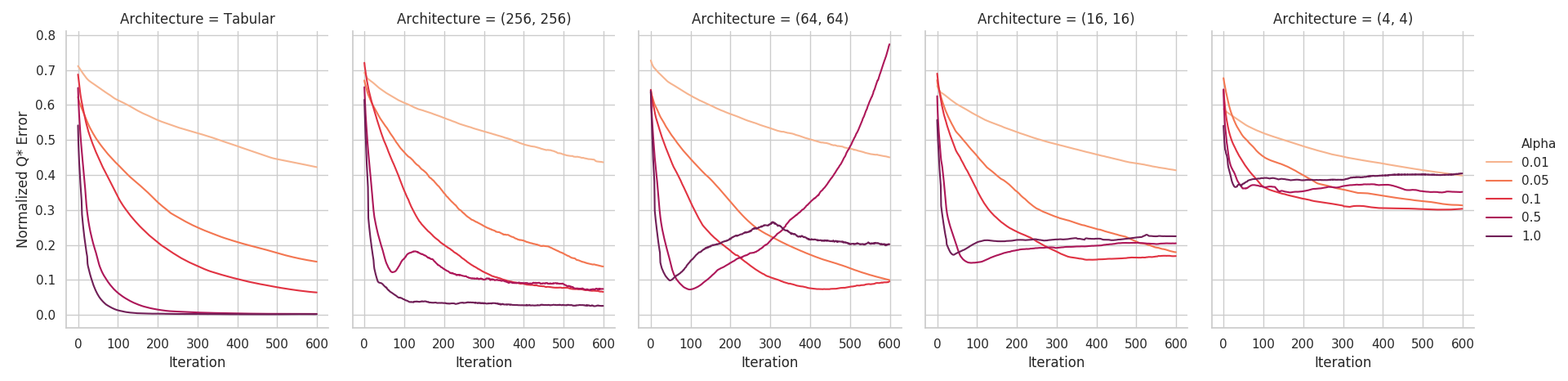}
\end{subfigure}%

\begin{subfigure}[t]{1.0\textwidth}
    \centering
    \includegraphics[width=0.99\textwidth]{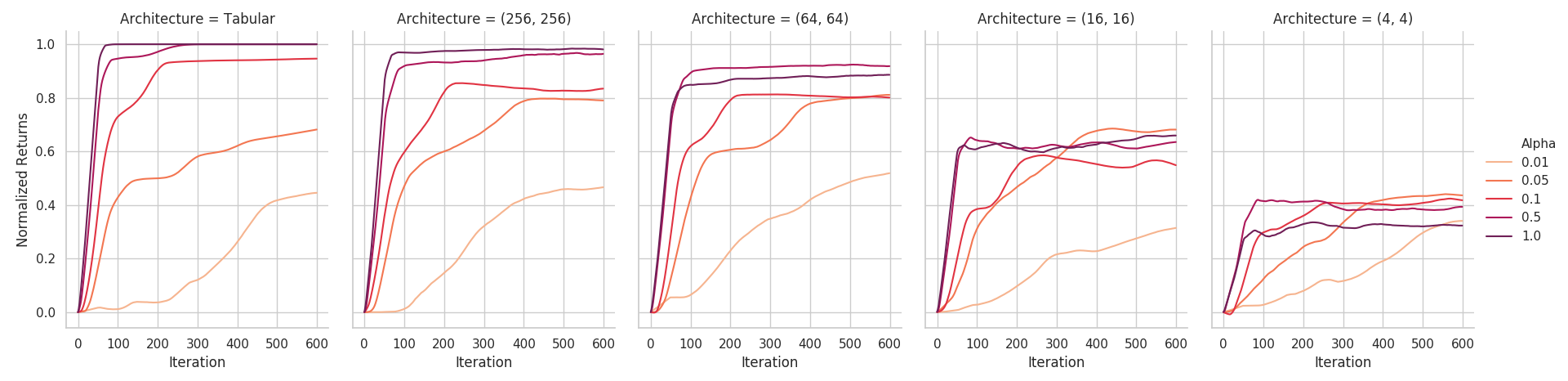}
\end{subfigure}%
\caption{\label{fig:smooth_fqi} Results for the $\alpha$-smoothed Bellman backup experiment. Normalized $\linfnorm$ norm error to $Q^*$ and normalized returns plotted for different values of $\alpha$ and architectures. Values are averaged over all domains and 5 seeds. For large architectures, higher values of $\alpha$ result in faster convergence and higher asymptotic returns. However, for smaller architectures, low values of $\alpha$ slightly outperform higher values.}
\end{figure*}

\section{Adversarial Feature Matching (AFM): Detailed Explanation and Practical Implementation}
\label{app:adversarial}
As described in section~\ref{sec:afm}, we devise a novel weighting scheme for the Bellman error objective based on an adversarial minimax game. The adversary computes weights $p_\phi(s, a)$ (representing the weighting distribution $\mu$), for the Bellman error: $(Q_{\theta, w}(s, a) - y(s, a))^2$. Recalling from Section~\ref{sec:afm}, the optimization problem is given by:
\begin{multline*}
    \min_{\theta, w} \max_{\phi} \mathbb{E}_{p_\phi(s, a)} [(Q_{w, \theta} (s, a) - y(s, a))^2]\\
s.t.~~ \big\vert\big\vert \mathbb{E}_{p_\phi(s, a)}[\Phi(s)] - \frac{\sum_i \Phi(s_i)}{N} \vert\vert \leq \varepsilon
\end{multline*}
where $\Phi(s)$ are the state features learned by the Q-function approximator. $\Phi(s)$ is easy to extract out of the multiheaded $Q(s,a)$ ($Q(s, a) = w_a^T\Phi(s)$) model typically used for discrete action control, as one choice is to let $\Phi(s)$ be the output of the penultimate layer of the Q-network. For continuous control tasks, however, we model $\Phi(s, a)$ (which is a function of the actions as well) as state-only features are unavailable, unless separately modeled. This can also be interpreted as modelling a feature matching constraint on the gradient of $Q(s, a)$ with respect to the last linear parameters $w_a$. A possible extension is to take into account the entire gradient as the features in the feature matching constraint, that is, $\nabla_{w, \theta} Q_{w, \theta}(s, a)$.

This choice of the constraint is suitable and can be interpreted in two ways. First, an adversary constrained in this manner has enough power to exploit the Q-network at states which get aliased under the chosen function class, thereby promoting more separable feature learning and reducing some negative aspects of function approximation that can arise in Q-learning. This is also similar in motivation to \citep{martha2018sparse}. Second, this feature constraint also bears a similarity the Maximum Mean Discrepancy (MMD) distance between two distributions $P(x)$ and $Q(x)$ that can be written as $\text{MMD}^2(P, Q) := ||\mathbb{E}_{P}[\Phi] - \mathbb{E}_{Q}[\Phi]||_{\mathcal{H}}$, where the set of functions $\Phi$ is the canonical feature map, $\Phi: \mathbb{R}^n \rightarrow \mathcal{H}$ (from real space to the RKHS). In our context, this is analogous to optimizing a distance between the adversarial distribution $p_\phi(s, a)$ and the replay buffer distribution $p_{rb}(s, a)$ (as the average is a Monte-Carlo estimator of the expected $\Phi$ under the replay buffer distribution $p_{rb}(s, a)$). In the light of these arguments, AFM, and other associated methods that take into account the properties of the function approximator into account (for example, $\Phi$ here), can greatly reduce the bias incurred due to function approximation in the due course of Q-learning/FQI, as depicted in \ref{fig:function_approx}.

\paragraph{Solving the optimization} We solve this saddle point problem using alternating dual gradient descent. We first solve the inner maximization problem, and then use its solution to then solve the outer minimization problem. We first compute the Lagrangian for the maximization, $\mathcal{L}_{\text{inner}}({\phi; \lambda, \theta})$ by introducing a dual variable $\lambda$,
\begin{multline*}
    \mathcal{L}_{\text{inner}}({\phi; \lambda, \theta}) =  -\mathbb{E}_{p_\phi(s, a)} [(Q_\theta (s, a) - y(s, a))^2] +\\ \lambda \big( \vert\vert \mathbb{E}_{p_\phi(s, a)}[\Phi(s)] - \frac{\sum \Phi(s)}{N} \vert\vert - \varepsilon \big)
\end{multline*}
(Note that this Lagrangian is flipped in sign because we first convert the maximization problem to standard minimization form.) We now solve the inner problem using dual gradient descent. We then plug in the solutions (approximate solutions obtained after gradient descent), $(p^*, \lambda^*)$ into the Lagrangian, to then solve the outside minimization over $\theta$. Note that while $\Phi$  depends on $\theta$ (as it is the feature layer of the Q-network), we don not backpropagate through $\Phi$ while solving the minimization. This improves stability of the Q-network training in practice and to makes sure that Q-function is only affected by FQI updates. In practice, we take up to 10 gradient steps for the inner problem every 1 gradient step of the outer problem. The algorithm is summarized in Algorithm~\ref{alg:afm}. Our results provided in the main paper and here don't particularly assume any other tricks like Optimistic Gradient~\citep{daskalakis2018training}, using exponential moving average of the parameters~\citep{yaz2018the}. Our tabular experiments seemed to benefit some what using these tricks.   

\begin{algorithm}
\caption{\label{alg:afm}AFM with Exact-FQI}
\begin{algorithmic}[1]
    \STATE Initialize Q-value approximator $Q_{\theta, w}(s,a)$, \textbf{projection distribution $\mu_{\phi}(s, a)$, threshold $\varepsilon$}
    \FOR{step $t$ in \{1, \dots, N\}}
    \STATE Initialize Q-value approximator $Q_{\theta, w}(s,a)$.
    \STATE Evaluate $Q_{\theta^t, w^t}(s,a)$ at all states.
    \STATE Compute exact target values at all states. \\
    $y(s,a) = r(s,a) + \gamma E_{s'}[ V_{\theta^t}(s')]$
    \STATE {Minimize the \emph{negative} projection loss with respect to $\phi$ subject to the feature $\Phi$ matching constraint exactly over all states and actions}
    \begin{multline*}
        \phi_{t+1} \leftarrow \arg \min_{\phi} -\mathbb{E}_{p_\phi}[ (Q_{\theta, w}(s, a) - y(s, a))^2] \\
        \text{s.t. } ||\mathbb{E}_{\mu}[\Phi(s, a)] - \frac{\Phi(s, a)}{N}|| \leq \varepsilon\\
    \end{multline*}
    \vspace{-10pt}
    Maximize the Dual Loss w.r.t. $\lambda$.
    \begin{multline*}
        \lambda_{t+1} \leftarrow \arg \max_{\lambda \geq 0} \lambda (||\mathbb{E}_{\mu}[\Phi(s, a)] - \frac{\Phi(s, a)}{N}|| - \varepsilon) 
    \end{multline*}
    \STATE Repeat Step 6 for K steps (K $\in [1, 10]$).
    \STATE Minimize projection loss with respect to $\mu$: \\
    $\theta^{t+1}, w^{t+1} \leftarrow \argmin{\theta, w} E_{p_\phi}[ (Q_{\theta, w}(s,a) - y(s,a))^2]$
    \ENDFOR
\end{algorithmic}
\end{algorithm}

\paragraph{Practical implementation with replay buffers} We incorporate this weighting/sampling distribution into Q-learning in the setting with replay buffers and with state-action sampling. We evaluate the \textbf{weighting version} of our method, AFM, where, we usually sample a large batch $B$ of state-action pairs from a usual replay buffer used in Q-learning, but use importance weights to then match $p_\phi(s,a)$ in expectation. Thus, we use a parametric function approximator to model $\frac{p_\phi(s, a)}{p_{rb}(s, a)}$ -- that is, the importance weights of the adversarial distribution with respect to the replay buffer distribution $p_{rb}(s, a)$. Mathematically, we estimate: $E_{p_\phi(s, a)}[\delta(s, a)] := E_{p_{rb}(s, a)}[\frac{p_\phi(s, a)}{p_{rb}(s, a)} \delta(s, a)]$, where $\delta(s, a) = (Q_{\theta, w} (s, a) - y(s, a))^2$. The latter expectation is then approximated using a set of finite samples. It has been noted in literature that importance sampling (IS) suffers from high variance especially if the number of samples is small. Hence, we use the self-normalized importance sampling estimator, which averages the importance weights in a set of samples or a large number of samples. That is, let $w_{p/p_{rb}} = \frac{p_\phi(s, a)}{p_{rb}(s, a))}$, then instead of using $w_{p/p_{rb}}$ as the importance weights, we use $\Tilde{w}_{p/p_{rb}}(x) = \frac{w_{p/p_{rb}}(x)}{\sum_{y \in B} w_{p/p_{rb}}(y)}$ (where $x$ and $y$ represent state-action tuples; concisely mentioned for visual clarity) as the importance weights. We also regularize the second-order Renyi Divergence between $p_{rb}$ and $p_\phi$ for stability. Mathematically, it can be shown that this is a lower bound on the true expectation of $\delta$ under $p_\phi$, which is being estimated using importance sampling. This result has also been shown in~\citep{metelli2018nips} (Theorem 4.1), where the authors use this lower bound in policy optimization via importance sampling. We state the theorem below for completeness.

\begin{theorem}
\textbf{\citep{metelli2018nips}} Let $P$ and $Q$ be two probability measures on the measurable space $(X , F)$ such that
$P << Q$ and $d_2(P ||Q) < +\infty$. Let $x_1, x_2, \cdots , x_N$ be i.i.d. random variables sampled from $Q$, and $f : X \rightarrow \mathbb{R}$ be a bounded function. Then, for any $0 < \delta \leq 1$ and $N > 0$ with probability at least $1 - \delta$ it holds that:
\begin{multline*}
    \mathbb{E}_{x \sim Q}[f(x)] \geq \frac{1}{N} \sum w_{P/Q}(x_i) f(x_i) -\\ ||f||_{\infty} \sqrt{\frac{(1 - \delta) d_2(P||Q)}{N\delta}}
\end{multline*}
where $d_2(P||Q) \propto \mathbb{E}_{Q} \big[ (\frac{P(x)}{Q(x)})^2 \big]$ is the exponentiated second-order Renyi Divergence between $P$ and $Q$.
\end{theorem}

Hence, our objective for the inner loop now becomes: $ \max_\phi \mathbb{E}_{p_\phi(s, a)} [\delta(s, a)] = \max_\phi \mathbb{E}_{p_{rb}}[\frac{p_\phi(s, a)}{p_{rb}(s, a)} \delta(s, a)]$ is now computed using samples with an additional renyi regularisation term. Since, we end up modeling this ratio,  $f_\phi(s, a) := \frac{p_\theta(s, a)}{p_{rb}(s, a)}$ through out parameteric model, we can hence easily compute an estimator for the Renyi divergence term. The overall lower bound inner maximization problem is:
\begin{multline*}
\max_\phi \frac{1}{N} \sum_{(s, a) \sim p_{rb}}[f_\phi(s, a) (Q_{\theta, w} (s, a) - y(s, a))^2] -\\ C \sqrt{\frac{(1 - \delta) (\frac{\sum f_\phi(s, a))^2}{N})}{N\delta}} \\
\text{s.t.~~} \vert\vert \frac{\sum_{s, a \in p_{rb}}[f_\phi(s, a) \Phi(s)]}{N} - \frac{\sum_{s,a \in p_{rb}} \Phi(s)}{N} \vert\vert  \leq \varepsilon
\end{multline*}

We found that this Renyi penalty helped stabilize training. In practice, we model the importance weights: $f_\phi(s, a)$ as a parametric model with an identical architecture to the Q-network. We use parameter clipping for $f_\phi(s, a)$, where the parameter are clipped to $[-0.1, 0.1]$, analogous to Wasserstein GANs~\citep{pmlr-v70-arjovsky17a}. We also found that self-normalization during importance sampling has a huge practical benefit. Note that as the true $\linfnorm$ norm of the Bellman error is not known, for computing $C$ in the Renyi Divergence term, and hence we either replace it by constant, or compute a stochastic approximation to the $\linfnorm$ norm over the current batch. We found the former to be more stable, and hence, used that in all our experiments. This coefficient of the Renyi divergence penalty is tuned uniformly between $[0.0, 0.25]$. The learning rate for the adversary was chosen to be 1e-4 for the tabular environments, and 5e-4 for TD3. The batch size for our algorithm was chosen to be 128 for the tabular environments and 500 for TD3/SAC. Note that a larger batch size ensures smoothness in the minmax optimization problem. We also found that instead of having a $1D$ Lagrange multiplier for the feature matching constraint, having $d$ Lagrange multipliers for constraining each of the individual dimensions of the features $\Phi \in \mathbb{R}^d$ also helps very much. This is to ensure that the hyperparameters remain the same across different architectures regardless of the dimension of the penultimate layer of the Q-network. The algorithm in this case is exactly the same as the algorithm before with a vector valued dual variable $\lambda$. We used TD3 and SAC implementations from rlkit (\url{https://github.com/vitchyr/rlkit/tree/master/rlkit})

\onecolumn{\section{Function approximation analysis on Mujoco Tasks}}
\label{appendix:sac_size_plots}
As discussed in Section \ref{sec:function_approx}, we validate our findings on the effect of function approximation on 3 MuJoCo tasks from OpenAI Gym with the SAC algorithm from the author's implementation at \cite{haarnoja2018sacapps}. We observe that bigger networks learn faster and better in general.
\begin{figure*}[ht]
    \begin{subfigure}[t]{0.30\textwidth}
        \centering
        \includegraphics[height=1.2in]{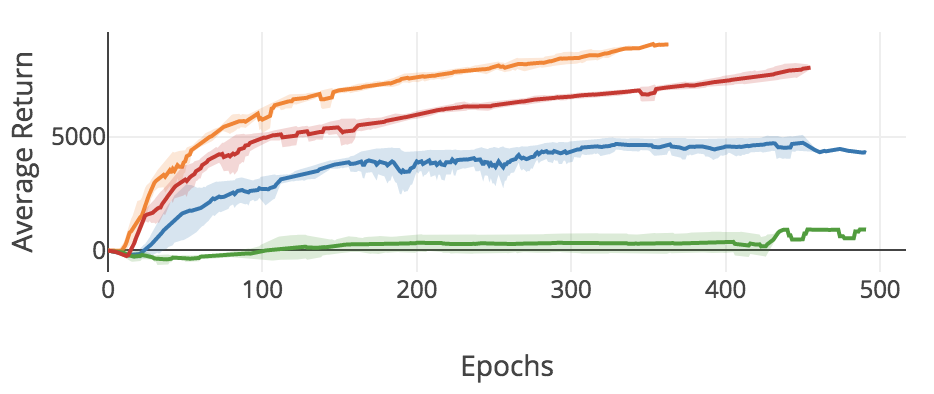}
        \caption{\centering{HalfCheetah-v2}}
    \end{subfigure}%
    \begin{subfigure}[t]{0.30\textwidth}
        \centering
        \includegraphics[height=1.2in]{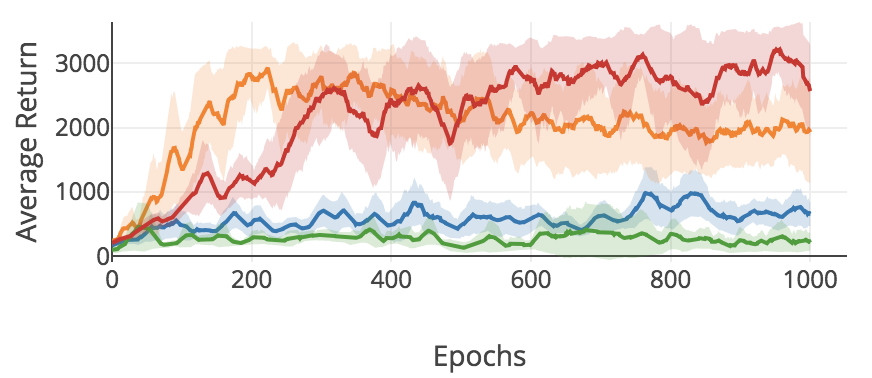}
        \caption{\centering{Hopper-v2}}
    \end{subfigure}
    \begin{subfigure}[t]{0.39\textwidth}
        \centering
        \includegraphics[height=1.2in]{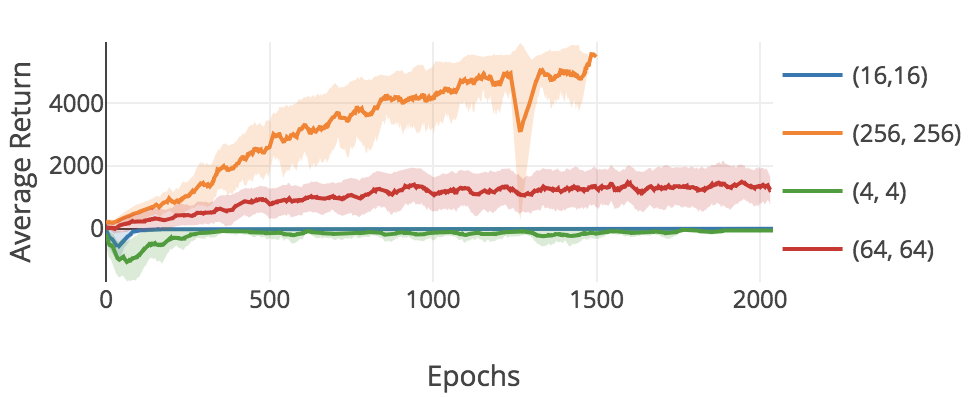}
        \caption{\centering{Ant-v2}}
    \end{subfigure}
    \caption{\label{fig:size_sac}Performance of different size architectures on 3 benchmark MuJoco tasks from OpenAI gym suite with the SAC algorithm. Values are averaged over 3 different seeds. A bigger network performs better in terms of learning speed and performance measured in terms of returns. Each epoch on the x-axis corresponds to 1000 environment steps.}
\end{figure*}

\section{Additional Plots}

\begin{figure*}[h]
    \centering
    \includegraphics[width=0.99\textwidth]{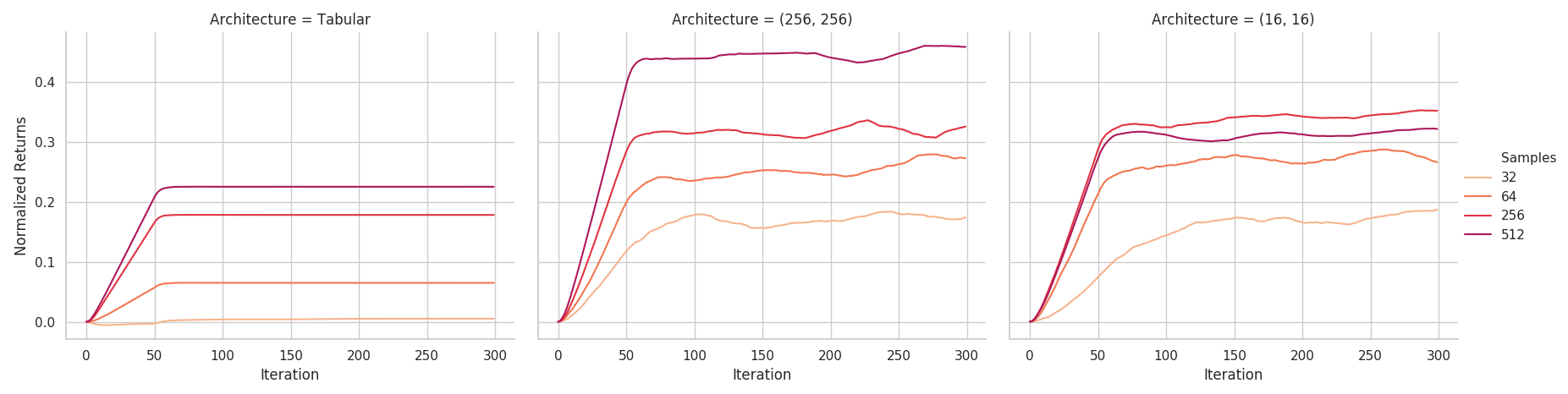}
\caption{\label{fig:sampling_arch_sweep} Normalized returns with Sampled-FQI, varying over architectures and number of on-policy samples.}
\end{figure*}

\vspace{-30pt}
\begin{figure*}[h]
    \centering
    \includegraphics[width=0.49\textwidth, scale=0.25]{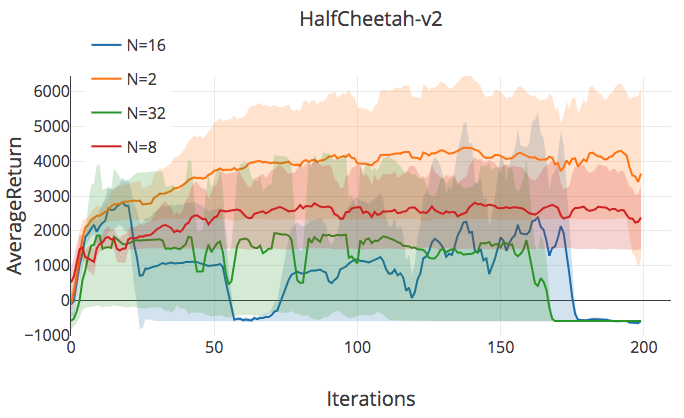}
    \includegraphics[width=0.49\textwidth, scale=0.25]{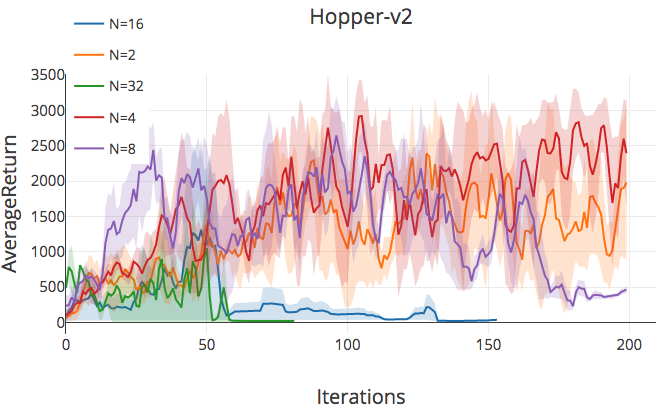}
\caption{\label{fig:td3_grad_sweep} Performance on Half Cheetah and Hopper trained via TD3 with replay buffer of size $2e4$ with increasing number of gradient steps taken per environment step ($N$) on the critic and the actor. Note the clearly observable decay in performance of the agent with more number of gradient steps -- which clearly validates our claim of the presence of overfitting in Q-functions. Each iteration on the x-axis corresponds to taking 5000 steps in the environment.}

\end{figure*}

\end{document}